\def\year{2021}\relax
\documentclass[letterpaper]{article} 
\usepackage{aaai21}  

\usepackage[switch]{lineno}
\usepackage{mathrsfs}
\usepackage{amsthm}
\newtheorem{definition}{Definition}
\newtheorem{lemma}{Lemma}
\newtheorem*{proof*}{Proof}
\newtheorem{theorem}{Theorem}
\newtheorem{corollary}{Corollary}
\usepackage{booktabs}
\usepackage{amssymb, amsmath}
\usepackage{mathtools}
\usepackage{times}  
\usepackage{helvet} 
\usepackage{courier}  
\usepackage[hyphens]{url}  
\usepackage{graphicx} 
\usepackage{natbib}  
\usepackage{caption} 
\newtheorem{assumption}{Assumption}
 \usepackage{xcolor}
 \usepackage{adjustbox}
\usepackage{array}
\usepackage{booktabs}
\usepackage{multirow}
\usepackage{pifont}

\newcommand*\rot{\rotatebox{90}}

\newcolumntype{R}[2]{%
    >{\adjustbox{angle=#1,lap=\width-(#2)}\bgroup}%
    l%
    <{\egroup}%
}
\ifodd 0
\newcommand{\newrev}[1]{{\color{blue}#1}} 
\newcommand{\rev}[1]{{#1}} 
\newcommand{\xr}[1]{{#1}} 
\newcommand{\revv}[1]{{#1}} 
\newcommand{\com}[1]{\textbf{\color{red}(COMMENT: #1)}} 
\newcommand{\mcom}[1]{\textbf{\color{purple}(MahdiCom: #1)}} 
\newcommand{\edt}[1]{\textbf{\color{magenta}#1}} 
\newcommand{\clar}[1]{\textbf{\color{green}(NEED CLARIFICATION: #1)}}
\else
\newcommand{\newrev}[1]{#1} 
\newcommand{\rev}[1]{#1}
\newcommand{\revv}[1]{#1}
\newcommand{\xr}[1]{#1}
\newcommand{\com}[1]{}
\newcommand{\mcom}[1]{}
\newcommand{\edt}[1]{}
\newcommand{\clar}[1]{}
\fi
\title{Improving Fairness and Privacy in   Selection Problems}
\author {
        Mohammad Mahdi Khalili,\textsuperscript{\rm 1}
        Xueru Zhang, \textsuperscript{\rm 2}
        Mahed Abroshan, \textsuperscript{\rm 3}
        Somayeh Sojoudi \textsuperscript{\rm 4}\\
}
\affiliations {
    \textsuperscript{\rm 1} CIS Department, University of Delaware, Newark, DE, USA\\
    \textsuperscript{\rm 2} EECS Department, University of Michigan, Ann Arbor, MI, USA\\
    \textsuperscript{\rm 3} Alan Turing Institute, London, UK\\
        \textsuperscript{\rm 4} EECS Department, University of California, Berkeley, CA, USA\\
    khalili@udel.edu, xueru@umich.edu, mabroshan@turing.ac.uk, sojoudi@berkeley.edu
}

\begin{document}

\maketitle

\begin{abstract}
\revv{Supervised learning models have been increasingly used for making decisions about individuals in applications such as hiring, lending, and college admission. These models may inherit pre-existing biases from training datasets and discriminate against protected attributes (e.g., race or gender). In addition to unfairness, privacy concerns also arise when the use of models reveals sensitive personal information. Among various privacy notions, differential privacy has become popular in recent years. In this work, we study the possibility of using a differentially private exponential mechanism as a post-processing step to improve both fairness and privacy of supervised learning models. Unlike many existing works, we consider a scenario where a supervised model is used to select a limited number of applicants as the number of available positions is limited. This assumption is well-suited for various scenarios, such as job application and college admission. We use ``equal opportunity'' as the fairness notion and show that the exponential mechanisms can make the decision-making process perfectly fair. Moreover, the experiments on real-world datasets show that the exponential mechanism can improve both privacy and fairness, with  a slight decrease in accuracy compared to the model without post-processing.}
\end{abstract}
\textbf{Area:}  ML: Ethics -- Bias, Fairness, Transparency \& Privacy

\section{Introduction}

Machine learning (ML) algorithms 
trained based on real-world datasets have been used in various decision-making applications (e.g., job applications and criminal justice). Due to the pre-existing bias in the datasets, the decision-making process can be biased against protected attributes (e.g., race and gender). 
For example, COMPAS (Correctional Offender Management Profiling for Alternative Sanctions) recidivism prediction tool used as part of inmate parole decisions by courts in the United States has been shown to have a substantially higher false positive rate on African Americans compared to \revv{White people} \cite{COMPAS}.  
\xr{In speech recognition, products} such as Amazon's Alexa and Google Home can have accent bias, with Chinese-accented and Spanish-accented English hardest to understand \cite{accent}. Amazon had been using automated software since 2014 to assess applicants’ resumes,  which was found to be biased against women \cite{amazon}.


There are various potential causes for \xr{such discrimination. Bias can be introduced when the data is collected. }
For instance, if a group (i.e., a majority group) contributes more to the dataset as compared to another group (i.e., a minority group), then the model trained based on the dataset could be biased in favor of the majority group, and this group may experience a higher accuracy. Even if the data collection procedure is unbiased, the data itself may be biased, e.g., the labels in the training dataset \rev{may exhibit bias} if they are provided based on an agent's opinion \cite{jiang2019identifying}. 

The fairness issue has been studied extensively in the literature.
\xr{A variety of fairness notions have been proposed to measure the unfairness, and they can be roughly classified into two families:}
 \textit{individual fairness} and \textit{group fairness}.  Individual fairness is in pursuit of equity in \xr{the} individual-level, that it requires any two similar individuals to be treated similarly \cite{biega2018equity,jung2019eliciting,gupta2019individual}. Group fairness aims to achieve a certain balance in \xr{the} group-level, that the population is partitioned into a small number of protected groups and it requires \revv{a} certain statistical measure (e.g.,  positive classification rates,  true positive rates,  etc.) to be approximately equalized across different protected groups \cite{zhang2019group,hardt2016equality,conitzer2019group,zhang2020long,zhang2020fair}.  \xr{There are mainly three approaches to improving fairness}
 \cite{zhang2020fairness}:   

\begin{itemize}
\item[i)] \textit{Pre-processing:} modifying the training datasets to remove the discrimination before training an ML model \cite{kamiran2012data,zemel2013learning}; 

\item[ii)] \textit{In-processing:} imposing certain fairness  criterion or modifying the loss 
function during 
the training process 
 \cite{agarwal2018reductions,zafar2019fairness};

\item[iii)] \textit{Post-processing:} altering the output of an existing algorithm to satisfy the fairness requirements \cite{hardt2016equality,Geoff}.
\end{itemize} 
In this work, we focus on \rev{group fairness} and use the \xr{post-processing} \rev{approach} to improving the fairness of a supervised \xr{learning} algorithm. 

In addition to unfairness issues, privacy concerns may incur
when making decisions based on \xr{individuals'} sensitive data.
Consider a lending 
\xr{scenario where  the loan approval decision is made based on an applicant's credit score.} 
The decision-making outcome can reflect the applicant’s financial situation and hence compromise his/her privacy. 

Various privacy-preserving techniques have emerged in recent years to protect individual privacy, \xr{such as} 
randomizing or anonymizing sensitive data \cite{liu2017personalized,sweeney2002k,zhu2004optimal,wang2012randomized}. Among them, differential privacy \cite{dwork2006differential} as a statistical notion of privacy has been extensively studied and deployed in practice. It ensures that no one, \xr{from an algorithm's outcome,} 
can infer with substantial higher confidence than random guessing whether a particular individual's data was included in the data analysis.
Various mechanisms have been developed to generate differentially private outputs such as \textit{exponential mechanism} \cite{mcsherry2007mechanism} and \textit{Laplace mechanism} \cite{dwork2006calibrating}. 

In this paper, we consider a scenario where a decision-maker (e.g., company) aims to accept $m\geq 1$ people from an applicant pool. This scenario can be referred to as a \textit{selection problem}. Each applicant has a hidden qualification state (e.g., capability for tasks) and observable features (e.g., GPA, interview performance). Suppose there is a supervised learning model that has been trained in advance and can be used for assigning each applicant a qualification score based on the features and queried by the decision-maker. These qualification scores can represent the likelihood that applicants are qualified, and the decision-maker selects $m$ applicants solely based on the qualification scores. During this process, privacy and fairness concerns may arise. As such, the decision maker's goal is to select $m$ applicants among those that are most likely to be qualified, and at the same time, preserve individual privacy and satisfy a fairness constraint.  \rev{Moreover, we allow the already trained supervised learning model to be a ``black-box"; the decision-maker has no access to the model parameters and can only use it to observe an applicant's score.} Within this context, we study the possibility of using an exponential mechanism as a post-processing scheme to improve both the privacy and fairness of the pre-trained supervised learning model. We consider ``equal opportunity'' as the notion of fairness\footnote{We discuss the generalization of our results to demographic parity in the appendix.}  and examine the relationship between fairness, privacy, and accuracy. 

\revv{\textbf{Related work.}} The most related works \xr{of} 
this paper \xr{are} 
\cite{hardt2016equality, jagielski2019differentially, cummings2019compatibility,mozannar2020fair,kleinberg2018selection}. Kleinberg and Raghavan \cite{kleinberg2018selection} \xr{study the effects of implicit bias on selection problems, and explore the role of the \textit{Rooney Rule} in this process.} 
 They show that the Rooney Rule can improve \xr{both the decision maker's utility and the disadvantaged group's representation. However, neither the fairness notion nor the privacy issues are considered in this work.}
 Hardt \textit{et al.} in \cite{hardt2016equality} introduce the notion of equal opportunity and develop a post-processing algorithm to improve the fairness of a supervised learning model. This work solely focuses on the fairness issue and does not provide any privacy guarantee. 
 Cummings \textit{et al.} in \cite{cummings2019compatibility} 
 study fairness in supervised learning and its relationship to differential privacy. In particular, they show that it is \textit{impossible} to train a differentially private classifier that satisfies exact (perfect) fairness and achieves a higher accuracy than a constant classifier. \revv{Therefore, many} works in the literature have focused on \xr{developing} approximately fair and differentially private algorithms. For instance, \cite{xu2019achieving} introduces an algorithm to train a differentially private logistic regression model that is approximately fair. 
 Jagielski \textit{et al.} \cite{jagielski2019differentially} develop post-processing and in-processing algorithms to train a classifier that satisfies approximate fairness and protects the privacy of protected attributes (not the training data).  \cite{mozannar2020fair} considers the notion of \textit{local} differential privacy and \xr{aims }
 to learn a \xr{fair} 
 supervised model \xr{from the data with the noisy and differentially private protected attributes. Similarly, }
 \cite{wang2020robust, kallus2019assessing, awasthi2020equalized} focus on fair learning  using noisy protected attributes \xr{but} without a privacy guarantee.

\textbf{Main contributions. }Most of the existing work 
aims to learn a model that minimizes the expected loss (e.g., classification error) over the entire population under certain fairness constraints. In settings such as hiring, lending, and college admission, it means the decision-maker should accept all the applicants as long as they are likely to be qualified. However, this may not be realistic for many real-world applications, when only a fixed number of positions are available, and only a limited number of applicants can be selected. In this paper, we shall consider this scenario where only a fixed number of people are selected among all applicants. Using equal opportunity as the fairness notion and differential privacy as the privacy notion, we identify sufficient conditions under which the exponential mechanism, in addition to a differential privacy guarantee, can also achieve \textit{perfect} fairness. Our results show that the negative result shown in \cite{cummings2019compatibility} (i.e., it is impossible to attain a perfectly fair classifier under differential privacy) does not apply to our setting when the number of acceptance is limited. In summary, our main contributions are as follows: 
\begin{itemize}
\item  We show that although the exponential mechanism has been designed and used mainly for preserving individual privacy, it is also effective in improving  fairness. The sufficient conditions under which the exponential mechanism can achieve \textit{perfect} fairness are identified.

\item We show that the accuracy of a supervised learning model
after using the exponential mechanism is monotonic in privacy leakage, which implies that the improvement of fairness and privacy is at the cost of accuracy. 
\item  Unlike \cite{cummings2019compatibility}, in our setting, we show that compared to other trivial algorithms (e.g., uniform random selection) that are perfectly fair, the exponential mechanism can achieve a higher accuracy while \xr{maintaining} perfect fairness. 
\end{itemize}
The remainder of the paper is organized as follows.  
We present our model in Section \ref{sec:Model}. The relation between fairness, privacy and accuracy is examined in Section \ref{sec:analysis}. The generalization to a scenario with $m$ available positions  is discussed in Section \ref{sec:choosingM}. 
We present the numerical experiments in Section \ref{sec:Num} and conclude the paper in Section \ref{sec:conclusion}.

\section{Model}\label{sec:Model}

Consider a scenario where $n$ individuals indexed by $\mathcal{N} = \{1,2,\cdots,n\}$ apply for some jobs/tasks. Each individual $i$ can be characterized by a tuple $(X_i,A_i,Y_i)$, where $Y_i\in \{0,1\}$ is the hidden qualification state representing whether $i$ is qualified $(Y_i = 1)$ for the position or not $(Y_i = 0)$, $X_i \in \mathcal{X}$ is the observable features and $A_i\in \{0,1\}$ is the protected attribute (e.g., race, gender) indicating the group membership of individual $i$.  Tuples $\{(X_i, A_i, Y_i)|i=1, \ldots, n\}$ are i.i.d. random variables following some distribution $\mathsf{F}$. We allow $X_i$ to be correlated with $A_i$, and it may include $A_i$ as well. The decision-maker observes the applicants' features and aims to select $m$ people that are most likely to be qualified and  satisfy certain privacy and fairness constraints. 
	
\textbf{Pre-trained model and qualification scores. } We assume there is a supervised learning model $r:\mathcal{X}\to \mathcal{R}$ that has been trained in advance and can be queried by the decision-maker.  It takes features of each applicant as input, and outputs a qualification score indicating the likelihood that the applicant is qualified. Let $\mathcal{R} \coloneqq \{\rho_1,\ldots, \rho_{n'}\} \subset [0,1]$ be a set of all possible values for  the qualification score, and define $R_i = r(X_i)$ as individual $i$'s qualification score. The higher $R_i$ implies individual $i$ is more likely to be qualified, i.e., $Y_i = 1$.  Note that $R_i$ depends on $A_i$ through $X_i$ since $X_i$ and $A_i$ are correlated, and $X_i$ may include $A_i$. \xr{Without loss of generality, let 
$\rho_1 = 0$ and $\rho_{n'} = 1$.} 

\textbf{Selection procedure. } Let $\mathbf{D} = (X_1, \ldots, X_n)$ be a database that includes all the applicants' features, and $D$ be its realization. ${D}$ is the only information that the decision-maker can observe about applicants. The decision-maker first generates all applicants' qualification scores using pre-trained model $r(\cdot)$, and then uses these scores  $(R_1, \ldots, R_n)$ to select $m$ individuals.  We first  focus on a case when $m=1$, i.e., only one applicant is selected, even though there could be more than one qualified applicant in the applicant pool.  The generalization to $m > 1$ is studied in Section \ref{sec:choosingM}.

For notational convenience, we further define tuple $(X,A,Y)$ as a random variable that also follows distribution $\mathsf{F}$, and $R = r(X)$. Denote $(x,a,y)$ as a realization of $(X,A,Y)$. Similar to \cite{hardt2016equality}, we assume 
$\mathsf{F}$ can be learned during the training process and is known to the decision-maker.

 \subsection{Differential privacy}\label{subsec:privacy}


Let $x_i\in \mathcal{X}$ be the observable features of individual $i$, 
 and $D=(x_1, x_2,\ldots, x_n)$ be a database which includes all individuals' data. Moreover, $\mathcal{D} = \{(\hat{x}_1,\hat{x}_2,\ldots,\hat{x}_n)| \hat{x}_i \in \mathcal{X} \}$ denotes the set of all possible databases. 
\begin{definition}[Neighboring Databases]
	Two databases $D = (x_1,\ldots, x_n)$ and $D' = (x'_1,\ldots, x'_n)$ are neighboring databases if they differ only in one data point, noted as $D\sim D'$,  i.e., $$
	\exists i \in \mathcal{N} \mbox{ s.t. } x_i\neq x'_i~ \mbox{and} ~ x_j=x'_j~ \forall j \neq i.
$$
\end{definition} 

\begin{definition}[Differential Privacy \cite{dwork2006differential}]
A randomized algorithm $\mathscr{M}$ is 
$\epsilon$-differentially private if for any two neighboring databases $D$ and $D'$ and for any possible set of output $\mathcal{W}\subseteq \text{Range}(\mathscr{M})$, it holds that
$
	\frac{{\Pr}\{\mathscr{M}(D)\in \mathcal{W}\}}{{\Pr}\{\mathscr{M}(D')\in \mathcal{W}\}} \leq \exp\{\epsilon\}.
$
\end{definition}
Privacy parameter $\epsilon\in [0,\infty)$ can be used to measure privacy leakage; the smaller $\epsilon$ corresponds to the stronger privacy guarantee. For sufficiently small $\epsilon$, the distribution of  output remains almost the same as a single data point in the database changes. It suggests that an attacker cannot infer the input data with high confidence after observing the output; thus, individual privacy is preserved. Next, we introduce a notable 
mechanism that can achieve differential privacy.
\begin{definition}[Exponential mechanism \cite{mcsherry2007mechanism}]
Denote $\mathcal{O}= \{o_1, \ldots, o_{\hat{n}}\}$ as the set of all possible outputs of algorithm $\mathscr{M}$, and $v:\mathcal{O}\times \mathcal{D}\rightarrow \mathbb{R}$ as a score function, where \rev{a higher value of $v(o_i,D)$ implies that output $o_i$ is more appealing under database $D$. Let $\Delta =  \max_{i,D\sim D'} |v(o_i,D)-v(o_i,D')|$ be defined as the sensitivity of score function.} 
Then, exponential mechanism  $\mathscr{M}:\mathcal{D}\rightarrow \mathcal{O}$ that satisfies  $\epsilon$-differential privacy  selects \xr{ $o_i \in \mathcal{O}$ } with  probability  
$
		{\Pr}\{\mathscr{M}(D) = o_i\} = \frac{\exp\{\epsilon \cdot \frac{v(o_i,D)}{2\Delta}\}}{\sum_{j=1}^{\hat{n}}\exp\{\epsilon \cdot \frac{v(o_j,D)}{2\Delta}\}}.
$

\end{definition}

 \subsection{Make selections using exponential mechanism}\label{subsec:selection} 
 Given a set of  qualification scores $(R_1,\ldots,R_n)$ generated from a pre-trained model $r(\cdot)$, the decision-maker selects an individual based on them, and meanwhile tries to preserve privacy with respect to database $\mathbf{D} = (X_1, \ldots, X_n)$. To this end, the decision-maker makes a selection using the exponential mechanism with score function $v: \mathcal{N} \times\mathcal{D} \rightarrow [0,1]$,  where $\mathcal{D} = \mathcal{X}^n$ is the set of all possible databases.
 
 One natural choice of the score function would be $v(i,D) = r(x_i)$, i.e., an applicant with a higher qualification score is more likely to be selected. Because $0\leq r(x) \leq 1$, \revv{for all} $x\in \mathcal{X}$, the sensitivity of score function $v(i,D)$ is $\Delta = \max_{i,D\sim D'} |v(i,D)-v(i,D')| = 1$.
 
 Let $\mathscr{A}_{\epsilon}: \mathcal{D}\rightarrow \mathcal{N}$ be an $\epsilon$-differentially private \newrev{exponential mechanism} used by the decision-maker to select one individual. Using $\mathscr{A}_{\epsilon}(\cdot)$, after observing  realizations of $(R_1,R_2,\ldots,R_n)$,  individual \revv{$i\in \mathcal{N}$} is selected with probability 
$\Pr\{\mathscr{A}_{\epsilon}(D) = i \} =   \frac{\exp\{\epsilon\cdot \frac{r_i}{2}\}}{\sum_{j \in \mathcal{N}} \exp\{\epsilon\cdot \frac{r_j}{2}\}},$
 where $r_i$ is the realization of random variable $R_i$. For each individual $i$, define a Bernoulli random variable $I_{i,\epsilon} \in \{0,1\}$ indicating whether $i$ is selected $(I_{i,\epsilon} =1)$ under algorithm $\mathscr{A}_{\epsilon}(\cdot)$ or not $(I_{i,\epsilon} =0)$. We have, 
\begin{eqnarray*}
	&&\Pr\{I_{i,\epsilon} =1\} \\ &=&\sum_{(r_1,\ldots, r_n) \in \mathcal{R}^n } \Pr\left\lbrace I_{i,\epsilon} =1 | \cap_{j=1}^n\{R_j = r_j\}\right\rbrace
	\cdot \prod_{j=1}^n f_R (r_j)\nonumber \\
	&=& \sum_{(r_1,\ldots, r_n) \in \mathcal{R}^n } \frac{\exp\{\epsilon\cdot \frac{r_i}{2}\}}{\sum_{j =1}^n \exp\{\epsilon\cdot \frac{r_j}{2}\} } \cdot \prod_{j=1}^n f_R (r_j),
\end{eqnarray*}
where $f_R(\cdot)$ is the probability mass function (PMF) of random variable $R$.
If we further define random variable 
$ Z_{i,\epsilon} = \frac{\exp\{\epsilon\cdot \frac{R_i}{2}\}}{\sum_{j \in \mathcal{N}} \exp\{\epsilon\cdot \frac{R_j}{2}\} }$ \newrev{and denote the expectation of $Z_{i,\epsilon}$ by $E\left\lbrace Z_{i,\epsilon}\right\rbrace$,}   
then 
$E\left\lbrace Z_{i,\epsilon}\right\rbrace = \Pr\left\lbrace I_{i,\epsilon} = 1\right\rbrace$ \xr{holds}.

\subsection{Fairness metric}\label{subsec:fairness} 
 Based on the protected attribute $A$, $n$ applicants can be partitioned into two groups. To measure the unfairness between two groups resulted from using algorithm $\mathscr{A}_{\epsilon}(\cdot)$, we shall adopt a group fairness notion. For the purpose of exposition, we focus on one of the most commonly used notions called \textit{equal opportunity} \cite{hardt2016equality}. The generalization to  \textit{demographic parity} fairness \cite{dwork2012fairness} is discussed in the appendix.
	
In binary classification, \textit{equal opportunity} fairness requires that the true positive rates experienced by different groups to be equalized, i.e., $\Pr\{\widehat{Y}=1|A=0,Y=1\}=\Pr\{\widehat{Y}=1|A=1,Y=1\}$, where $\widehat{Y}$ is the predicted label by the classifier. In our problem when the number of acceptance is $m=1$, this definition can be adjusted as follows,   
\begin{definition}[Fairness metric]\label{def:fairness}
	Consider an algorithm $\mathscr{M}:\mathcal{D}\to \mathcal{N}$ that selects one individual from $n$ applicants. Given database $\mathbf{D}$ and $\mathscr{M}(\cdot)$, for all $i\in \mathcal{N}$ define a Bernoulli random variable  $K_i$ such that $K_i = 1$ if \xr{$\mathscr{M}(\mathbf{D}) = i$} and  $K_i = 0$ \revv{otherwise}.  Then algorithm $\mathscr{M}(\cdot)$ is $\gamma$-fair if
	{\small
		\begin{eqnarray}
		\Pr\{K_{i} =1 | A_i=0,Y_i = 1\} - \Pr\{ K_{i} = 1| A_i = 1, Y_i = 1\} = \gamma\revv{.} \nonumber  
		\end{eqnarray}}
\end{definition}
	Note that $-1\leq \gamma\leq 1$, and negative (\xr{resp. }positive) $\gamma$ implies that algorithm $\mathscr{M}(\cdot)$ is biased in favor of the group with protected attribute $A=1$ (\xr{resp. }$A = 0$). In particular, we say $\mathscr{M}(\cdot)$ is \textbf{\textit{perfectly fair}} if  $\gamma = 0$.\footnote{Note that if algorithm $\mathscr{M}(\cdot)$ selects an individual solely based on i.i.d qualification scores $(R_1,\ldots,R_n)$ and does not differentiate individuals based on their indexes, then  $ \Pr\{K_{i} =1 | A_i=0,Y_i = 1\} - \Pr\{ K_{i} = 1| A_i = 1, Y_i = 1\} =   \Pr\{K_{j} =1 | A_j=0,Y_j = 1\} - \Pr\{ K_{j} = 1| A_j = 1, Y_j = 1\}, \forall i,j$.  }

For algorithm $\mathscr{A}_{\epsilon}(\cdot)$ in Section \ref{subsec:selection} that selects an individual using the exponential mechanism, $\gamma$ in above definition can be equivalently written as 
	{\small	\begin{equation}\label{def:Fairness}
	E\left\lbrace  Z_{i,\epsilon } |A_i=0,Y_i =1 \right\rbrace -E\left\lbrace  Z_{i,\epsilon } |A_i=1 ,Y_i =1 \right\rbrace = \gamma .
	\end{equation}}
\vspace{-0.75cm} 
\subsection{Accuracy metric}\label{sec:Accuracy}
\revv{It is easy to develop trivial algorithms that are both $0$-differentially private and $0$-fair. For example, $\mathscr{A}_{0}(\cdot)$ which selects one individual randomly and uniformly from $\mathcal{N}$, or a deterministic algorithm that always selects a particular individual. However, 
 \xr{both} algorithms do not use qualification scores 
  \xr{to make decisions.} 
The primary goal of the decision-maker  
\xr{that} selecting \xr{the} most qualified individuals is \xr{thus} undermined. 
We need to introduce another metric to evaluate the ability of $\mathscr{A}_{\epsilon}(\cdot)$ to select qualified individuals. 
}
\begin{definition}[Accuracy]\label{def:accuracy} An algorithm $\mathscr{M}:\mathcal{D}\to \mathcal{N}$ that selects one individual from $n$ applicants is $\theta$-accurate if 
	\begin{equation}\label{eq:accuracy}
	\Pr\{Y_{\mathscr{M}(\mathbf{D})} = 1\} = \theta.
	\end{equation}
\end{definition} 
As an example, for algorithm $\mathscr{A}_{0}(\cdot)$ that selects one individual uniformly at random from $\mathcal{N}$, it is $\theta$-accurate with $\theta = \Pr\{Y_{\mathscr{A}_{0}(\mathbf{D})}=1\} =\Pr\{Y=1\} $. The goal of this work is to examine whether it is possible to use exponential mechanism to achieve both differential privacy and the \textit{perfect} fairness, while maintaining a sufficient level of accuracy. In the next section, we study the relation between fairness $\gamma$, privacy $\epsilon$, and accuracy $\theta$ under \xr{$\mathscr{A}_{\epsilon}(\cdot)$}.

\section{Analysis}\label{sec:analysis}
\subsection{Fairness-Privacy Trade-off}\label{sec:fairnessPrivacy}
To study the relation between the fairness and privacy, let $ \gamma(\epsilon) = E\{Z_{i,\epsilon} | A_i = 0, Y_i =1\}- E\{Z_{i,\epsilon} | A_i = 1, Y_i =1\}$. Note that $\gamma(\epsilon)$ does not depend on $i$ as individuals are i.i.d.  Let $\mathscr{A}(\cdot)$ be the algorithm that selects an individual with the highest qualification score and breaks ties randomly and uniformly if more than one individual have the highest qualification score. Define a set of Bernoulli random variables $\{I_i\}_{i=1}^n$ indicating whether individual $i$ is selected $(I_i = 1)$ or not $(I_i = 0)$ under algorithm $\mathscr{A}(\cdot)$. Let ${N}_{\max}=|\{i\in \mathcal{N}| R_i = \max_j R_j\}|$ be  the number of individuals who have the highest qualification score, then
\begin{eqnarray}
\Pr\{I_i =1 \} = \sum_{k=1}^n \frac{1}{k} \cdot  \Pr\{R_i = \max_j R_j, {N}_{\max}= k \}.
\end{eqnarray}
Define 
\begin{eqnarray}
Z_i = \left\lbrace\begin{array}{ll}
0 &\mbox{if}~ R_i \neq \max_{j} R_j\\
\frac{1}{{N}_{\max}}&o.w.
\end{array}\right. .\end{eqnarray} 
Then it holds that $E\{Z_i\} = \Pr\{I_i =1\}$. The following \revv{lemma} characterizes the relation between random variables $Z_i$ and $Z_{i,\epsilon}$, \revv{which is essential to prove the next theorem.} 
\begin{lemma}[Sure convergence]\label{lem:convegance} Consider two algorithms $\mathscr{A}_{\epsilon}(\cdot)$ and $\mathscr{A}(\cdot)$ and the corresponding random variables $Z_{i,\epsilon}$ and $Z_i$. \revv{The following statements are true,}
	
1. $Z_{i,\epsilon}$ converges surely towards $Z_i$ as $\epsilon \to +\infty$. 

	2. $\mathscr{A}(\cdot)$ is $\gamma_{\infty}$-fair with $\gamma_\infty = \lim_{\epsilon\to+\infty} \gamma(\epsilon)$, i.e.,
	{\scriptsize
		\begin{eqnarray}
		&&\lim_{\epsilon \to +\infty}	E\left\lbrace Z_{i,\epsilon} | A_i = 0,Y_i =1\right\rbrace - E\left\lbrace Z_{i,\epsilon} | A_i = 1,Y_i =1\right\rbrace   \nonumber \\&=&	E\left\lbrace\lim_{\epsilon \to +\infty} Z_{i,\epsilon} | A_i = 0,Y_i =1  \right\rbrace - E\left\lbrace \lim_{\epsilon \rightarrow +\infty} Z_{i,\epsilon} | A_i = 1,Y_i =1  \right\rbrace  \nonumber \\
		&=& E\left\lbrace Z_{i} | A_i = 0,Y_i =1  \right\rbrace - E\left\lbrace  Z_{i} | A_i = 1,Y_i =1  \right\rbrace. \nonumber 
		\end{eqnarray}}
\end{lemma}
Lemma \ref{lem:convegance} implies that $\lim_{\epsilon \rightarrow +\infty} \gamma(\epsilon)=\gamma_{\infty}$ exists. It shows that  $\mathscr{A}_{\epsilon}(\cdot)$ using exponential mechanism is equivalent to algorithm $\mathscr{A}(\cdot)$ as $\epsilon \rightarrow \infty$. \revv{ In the next theorem, we identify a sufficient condition under which the exponential mechanism can achieve \textit{perfect} fairness with non-zero privacy leakage.}

\begin{theorem}\label{theo:fair}
\revv{There exists} $\epsilon_o >0$ such that $\gamma(\epsilon_o) = 0$ under $\mathscr{A}_{\epsilon_o}(.)$ if both of the following \revv{constraints} are satisfied:

 (1) $E\{Z_i|A_i = a,Y_i = 1 \}  < E\{Z_i|A_i = \neg a,Y_i = 1   \}$,
 
 (2) $ E \left\lbrace R_i | A_i = a, Y_i=1\right\rbrace >  E \left\lbrace R_i | A_i =  \neg a, Y_i=1\right\rbrace$,
\end{theorem}  
 where $a\in\{0,1\}$ and $\neg a = \{0,1\}\setminus a$.

\xr{Constraint} (1) above suggests that the applicants with protected attribute $A= \neg a$ are more likely to be selected than those with $A=a$ under algorithm $\mathscr{A}(\cdot)$; \xr{Constraint} (2) implies that \revv{on} average the applicants with protected attribute $A=a$ have \revv{a} higher qualification score than those with $A= \neg a$. These \xr{constraints} may be satisfied when the applicants with $A= \neg a$, as compared to those with $A=a$, have the smaller mean but much larger variance in their qualification cores. In \xr{Sec.} \ref{sec:Num}, we will show FICO credit score dataset \cite{reserve2007report} satisfies those \xr{constraints} for certain social groups.

It is also worth noting that perfect fairness is not always attainable under the exponential mechanism. In the next theorem, we identify sufficient conditions under which it is impossible to achieve the \textit{perfect} fairness using exponential mechanism unless the privacy guarantee is trivial $(\epsilon = 0)$.
\begin{theorem}\label{theo:fair2}
\revv{Let $f^a(\rho) \coloneqq \Pr\{R=\rho|A=a,Y=1 \}$. 
If $f^0(\rho) - f^1(\rho) > f^0(\rho') - f^1(\rho')$ and ${f}_R(\rho) < {f}_R(\rho')$  for all $\rho > \rho'$, \newrev{and $f^0(\rho) - f^1(\rho)\geq 0$  for $\rho = \rho_2,\ldots, \rho_{n'}$}, then we have, $\newline$
	1. $\gamma(\epsilon)> 0$ for $\epsilon>0$,} i.e., $\mathscr{A}_\epsilon(\cdot)$ is always biased in favor of individuals with protected attribute $A=0$. $\newline$
	2. $\gamma(\epsilon)<  \gamma_{\infty}, ~ \forall \epsilon\geq 0$, i.e., $\mathscr{A}_\epsilon(\cdot)$ is always fairer than $\mathscr{A}(\cdot)$. 
\end{theorem}
 The first condition implies that among applicants who are qualified, individuals with $A=0$ are more likely to have higher qualification scores as compared to individuals with protected attribute $A=1$. The second condition implies that most of the applicants have small qualification scores.   
Under these conditions, an exponential mechanism with $\epsilon>0$ can never achieve perfect fairness. Moreover, it shows that \xr{the exponential mechanism can improve fairness compared to the non-private algorithm $\mathscr{A}(\cdot)$ selecting an individual with the highest score. }

Theorems \ref{theo:fair} and \ref{theo:fair2} together show that the exponential mechanism may or may not achieve  \textit{perfect} fairness. Nevertheless, we can show that always there exists privacy parameter $\bar{\epsilon}$ such that $\mathscr{A}_{\bar{\epsilon}}(\cdot)$ is fairer than non-private $\mathscr{A}(\cdot)$.

\begin{theorem}\label{theo:nInfinity}
	If $\mathscr{A}(\cdot)$ is not 0-fair, 
	then there exists $\hat{\epsilon}\in (0,+\infty) $ such that $|\gamma({\epsilon})| < |\gamma_\infty|, \forall \epsilon \in (0,\hat{\epsilon})$. 
\end{theorem}
 This section has studied the possibility of using the exponential mechanism to improve both fairness and privacy. Note that even when perfect fairness is attainable, the outcome may not be desirable to the decision-maker if it is not accurate enough. In the next section, we shall take accuracy into account and examine its relation with privacy and fairness. 

\subsection{Accuracy-Privacy Trade-off}
Let $\theta(\epsilon) = \Pr\{ Y_{\mathscr{A}_{\epsilon}(\mathbf{D}) }= 1\}$ be the accuracy of $\mathscr{A}_{\epsilon}(\cdot)$. 
\begin{eqnarray}
\theta(\epsilon) &=& \sum_{i=1}^n \Pr\{ Y_i = 1 , I_{i,\epsilon} =1\} \nonumber\\
&=& \sum_{i=1}^n \Pr\{ I_{i,\epsilon} =1|Y_i = 1 \} \cdot \Pr\{Y_i = 1 \}\nonumber\\
&=& \Pr\{Y = 1\} \sum_{i=1}^n E\{Z_{i,\epsilon}  | Y_i=1\}. \nonumber
\end{eqnarray}
Therefore, accuracy maximization  is equivalent  to maximizing $E\{Z_{i,\epsilon} | Y_i=1\}$. Since $Z_{i,\epsilon}$ converges surely to $ Z_i$, similar to Lemma \ref{lem:convegance}, we can show that $\lim_{\epsilon \rightarrow +\infty} \theta(\epsilon)$ exists and is equal to the accuracy of non-private algorithm $\mathscr{A}(\cdot)$.   
We further make the following assumption that has been widely used in the literature \cite{jung2020fair, FairCake}.
	\begin{assumption}\label{assump1}
		$\frac{\Pr\{ R = \rho| Y = 1\} }{\Pr\{ R = \rho |Y = 0\} } \geq \frac{\Pr\{ R= \rho'| Y = 1\} }{\Pr\{ R = \rho' |Y = 0\} }  $, $\forall \rho > \rho'$.
	\end{assumption}
	Assumption \ref{assump1}, also known as the monotone likelihood ratio property of two PMFs $\Pr\{ R = \rho| Y = 1\}$ and $\Pr\{ R = \rho| Y = 0\}$, is relatively mild and can be satisfied by various probability distributions including Binomial and Poisson distributions. \rev{It implies that a qualified individual is more likely to have a high qualification score.}  The next theorem characterizes the effect of privacy parameter $\epsilon$ on the accuracy of $\mathscr{A}_{\epsilon}(\cdot)$.  
	\begin{theorem}\label{theo:acc}
		Under Assumption \ref{assump1}, $\theta(\epsilon)$ is increasing in $\epsilon$. 
	\end{theorem}
Suppose that the task is to make a selection such that unfairness and privacy leakage are less than or equal to $\gamma_{\max}$ and $\epsilon_{\max}$, respectively. Then, exponential mechanism  $\mathscr{A}_{\epsilon^*}(\cdot)$ has the highest accuracy, where $\epsilon^*$ is the solution to  \eqref{eq:opt}.
\begin{eqnarray}\label{eq:opt}
\epsilon^* = \arg\max_{\epsilon\leq \epsilon_{\max}} \theta(\epsilon), ~~ s.t. ~~|\gamma(\epsilon)|\leq\gamma_{\max}. 
\end{eqnarray}
Based on Theorem \ref{theo:acc}, 
\xr{we have} the following corollary. 

\begin{corollary}\label{corollary}
	Under Assumption \ref{assump1}, 
\xr{	$$\epsilon^* =\begin{cases}
	\epsilon_{\max},~~ \text{ if }~~|\gamma(\epsilon_{\max})| \leq \gamma_{\max}\\
	\max \{\epsilon\leq \epsilon_{\max} | \gamma_{\max} =|\gamma(\epsilon)| \}, ~~~o.w.
	\end{cases}
	 $$}
	
\end{corollary}

We conclude this section by comparing our results with \cite{cummings2019compatibility}. Consider a constant algorithm that always selects the first individual. The accuracy of this algorithm is given by $\Pr(Y_1 = 1)$, which is equal to the accuracy of algorithm $\mathscr{A}_0(\cdot)$, i.e., $\Pr(Y=1) = \theta(0)$. Moreover, the constant algorithm is perfectly fair. 
 If there exists \xr{$\epsilon_o >0$} such that $\mathscr{A}_{\epsilon_o}(\cdot)$ is perfectly fair, then \xr{according to Theorem \ref{theo:acc}}, the accuracy of $\mathscr{A}_{\epsilon_o}(.)$ would be larger than that of \xr{$\mathscr{A}_0(\cdot)$}, \xr{implying that the perfect fair $\mathscr{A}_{\epsilon_o}(\cdot)$ is also more accurate than}
 the constant algorithm. \xr{In contrast}, Cummings \textit{et al.} in \cite{cummings2019compatibility} conclude that perfect (exact) fairness is not compatible with differential privacy. \xr{Specifically, they show that} any differentially private classifier which is perfectly fair would have lower accuracy than 
 a constant classifier. \xr{The reason that our conclusion differs from \cite{cummings2019compatibility} is as follows. \cite{cummings2019compatibility} studies a classification problem where there is a hypothesis class $\mathcal{H}$ (i.e., a set of possible classifiers) and it aims to select a  \textit{perfect} fair classifier randomly  from $\mathcal{H}$ with high accuracy using a differentially private algorithm. In particular, they show that if $h$ is perfectly fair and more accurate than a constant classifier under database $D$, then there  exists another database $D'$ with $D'\sim D$ such that $h$ violates perfect fairness under $D'$. Because $h$ is selected with zero probability under $D'$, differential privacy is violated, which implies their negative results.   In contrast, we focus on a selection problem with  a fixed number of approvals using \textit{fixed} supervised learning model $r(\cdot)$. In this case, privacy and perfect fairness can be compatible with each other.  }

\section{Choosing more than one applicant}\label{sec:choosingM}
Our results so far are concluded under the assumption that only one applicant is selected. In this section, we extend our results to a scenario where $m>1$ applicants are selected.

To preserve individual privacy, an exponential mechanism is adopted to select $m$ individuals  from $n$ applicants based on their qualification scores. Let $\mathcal{S} = \{ \mathcal{G} |\mathcal{G} \subseteq \mathcal{N},~|\mathcal{G}| = m\}$ be the set of all possible selections, and we index the elements in $\mathcal{S}$ by $\mathcal{G}_1, \mathcal{G}_2, \ldots, \mathcal{G}_{{n \choose m}}$. Let $\mathscr{B}_{\epsilon}(\cdot)$ be the exponential mechanism that selects $m$ individuals from $\mathcal{N}$ and satisfies $\epsilon$-differential privacy. One choice of  score function $v: \mathcal{S}\times \mathcal{D}\rightarrow [0,1]$ is $v(\mathcal{G}_i,\mathbf{D}) =  \frac{1}{m}\sum_{j\in \mathcal{G}_i} R_j$, representing the averaged qualification of selected individuals in $\mathcal{G}_i$.\footnote{The generalization of our results to other types of score function will be discussed in the appendix.} The sensitivity of $v(\cdot,\cdot)$ is \revv{$\max_{\mathcal{G}\in \mathcal{S},D\sim D'} |v(\mathcal{G},D) - v(\mathcal{G},D')|  = \frac{1}{m}$.} That is, under algorithm $\mathscr{B}_{\epsilon}(\cdot)$, $\mathcal{G}_i$ is selected according to probability
\begin{equation}
	\Pr\{ \mathscr{B}_{\epsilon}(D) = \mathcal{G}_i\} = \frac{\exp\{\epsilon \cdot  \frac{\sum_{j\in \mathcal{G}_i} r_j}{2}\}}{\sum_{\mathcal{G}\in \mathcal{S}} \exp\{\epsilon \cdot \frac{\sum_{j\in \mathcal{G}} r_j}{2}\}}.
\end{equation}
Further define $\mathcal{S}_i =\{\mathcal{G} | \mathcal{G}\in \mathcal{S}, i\in \mathcal{G} \}$ as the set of all selections that contain  individual $i$. Define random variable 
\begin{equation}
W_{i,\epsilon}  = \sum_{\mathcal{G} \in \mathcal{S}_i} \frac{\exp\{\epsilon \cdot  \frac{\sum_{j\in \mathcal{G}} R_j}{2}\}}{\sum_{\mathcal{G}'\in \mathcal{S}} \exp\{\epsilon \cdot \frac{\sum_{j\in \mathcal{G}'} R_j}{2}\}}
\end{equation}
and Bernoulli random variable $J_{i,\epsilon}$ indicating whether $i$ is selected $(J_{i,\epsilon}=1)$ under $\mathscr{B}_{\epsilon}(\cdot)$ or not $(J_{i,\epsilon}=0)$. We have $\Pr\{J_{i,\epsilon} =1 \} = E\{ W_{i,\epsilon}\}$. Similar to Section \ref{sec:Model}, we further introduce algorithm $\mathscr{B}(\cdot)$ which selects a set of $m$ individuals with the highest average  qualification score. If there are more than one set with the  highest average qualification score, $\mathscr{B}(\cdot)$ selects one set among them uniformly at random. Let $\mathcal{S}_{\max} = \{ \mathcal{G}'\in\mathcal{S}| \sum_{j \in \mathcal{G}'}R_j = \max_{\mathcal{G}\in\mathcal{S} } \sum_{i \in \mathcal{G}}R_i   \}$. Each element in $\mathcal{S}_{\max}$ is a set of $m$ individuals who have the highest qualification scores in total. Define random variable
\begin{eqnarray}
W_i = \left\lbrace \begin{array}{ll}
0 & \mbox{ if } \mathcal{S}_i \cap \mathcal{S}_{\max} = \emptyset \\ 
\frac{1}{|\mathcal{S}_{\max} |} & o.w.
\end{array}\right. \nonumber
\end{eqnarray} 
and Bernoulli random variable $J_{i}$ indicating whether $i$ is selected $(J_{i}=1)$ under $\mathscr{B}(D)$ or not $(J_{i}=0)$. We have $\Pr\{J_{i} =1 \} = E\{ W_{i}\}$. Similar to the fairness metric  defined in Definition \ref{def:fairness}, we say algorithm $\mathscr{B}_{\epsilon}(\cdot)$ is $\gamma$-fair if the following holds,
\begin{equation}
E\{W_{i,\epsilon} | A_i=0, Y_i=1\} - E\{W_{i,\epsilon} | A_i=1, Y_i=1\} = \gamma ~.\nonumber 
\end{equation}
Re-write $\gamma$ above as $\gamma(\epsilon)$, a function of $\epsilon$. Similar to Lemma \ref{lem:convegance}, we can show that $W_{i,\epsilon}$ converges surely to $W_i$ as $\epsilon \to +\infty$, and that $\lim_{\epsilon \to +\infty}\gamma(\epsilon)$ exists. Moreover, $\mathscr{B}(D)$ is $\gamma_{\infty}$-fair with $\gamma_{\infty} = \lim_{\epsilon \to +\infty}\gamma(\epsilon)$. Next theorem identifies sufficient conditions under which exponential mechanism $\mathscr{B}_{\epsilon}(.)$ can be perfectly fair.
\begin{theorem}\label{theo:M}
	\revv{There exists} $\epsilon_o >0$ such that $\gamma(\epsilon_o) = 0$ under $\mathscr{B}_{\epsilon}(.)$ if both of the following \revv{constraints} are satisfied:
	
	(1) $E\{W_i| A_i = a,Y_i =1  \} <E\{W_i| A_i = \neg a,Y_i =1 \}$;
	
	(2)  $E\left\lbrace R_i | A_i = a,Y_i =1\right\rbrace > E\left\lbrace R_i | A_i = \neg a,Y_i =1\right\rbrace$.
\end{theorem}

To measure the accuracy in this scenario, we adjust Definition \ref{def:accuracy} accordingly. For each individual $i$, we define random variable $U_i = \begin{cases}
1, \text{ if } i\in \mathscr{B}_{\epsilon}(\mathbf{D})  \text{ \& }Y_i = 1\\
0,~~o.w.
	\end{cases}$ as the utility gained by the decision-maker from individual $i$, i.e., \revv{the decision-maker receives benefit $+1$ if accepting a qualified applicant and $0$ otherwise.} 
	Then the accuracy of $\mathscr{B}_{\epsilon}(\cdot)$ can be defined as the expected utility received by decision-maker, i.e., 
	\begin{eqnarray}\label{eq:acc}
\theta(\epsilon) = E\Big\{\frac{1}{m}\sum_{i\in \mathcal{N}} U_i\Big\}  
= \frac{1}{m}\sum_{i\in \mathcal{N}} \Pr\{J_{i,\epsilon} = 1,Y_i = 1 \}.
	\end{eqnarray}
Note that $(1/m)$ in \eqref{eq:acc} is a normalization factor to make sure $\theta(\epsilon) \in [0,1]$. Also, it is worth mentioning that $\theta(\epsilon)$ reduces to Definition \ref{def:accuracy} when $m=1$. Under Assumption \ref{assump1}, we can show that $\theta(\epsilon)$ is increasing in $\epsilon$, and an optimization problem similar to optimization \eqref{eq:opt} can be formulated given $\epsilon_{\max}$ and $\gamma_{\max}$ to find  appropriate $\epsilon^*$.  

\section{Numerical Experiments}\label{sec:Num}
\begin{figure*}[t]
	\centering
	\begin{minipage}{.32\textwidth}
		\centering
		\includegraphics[width=1\linewidth]{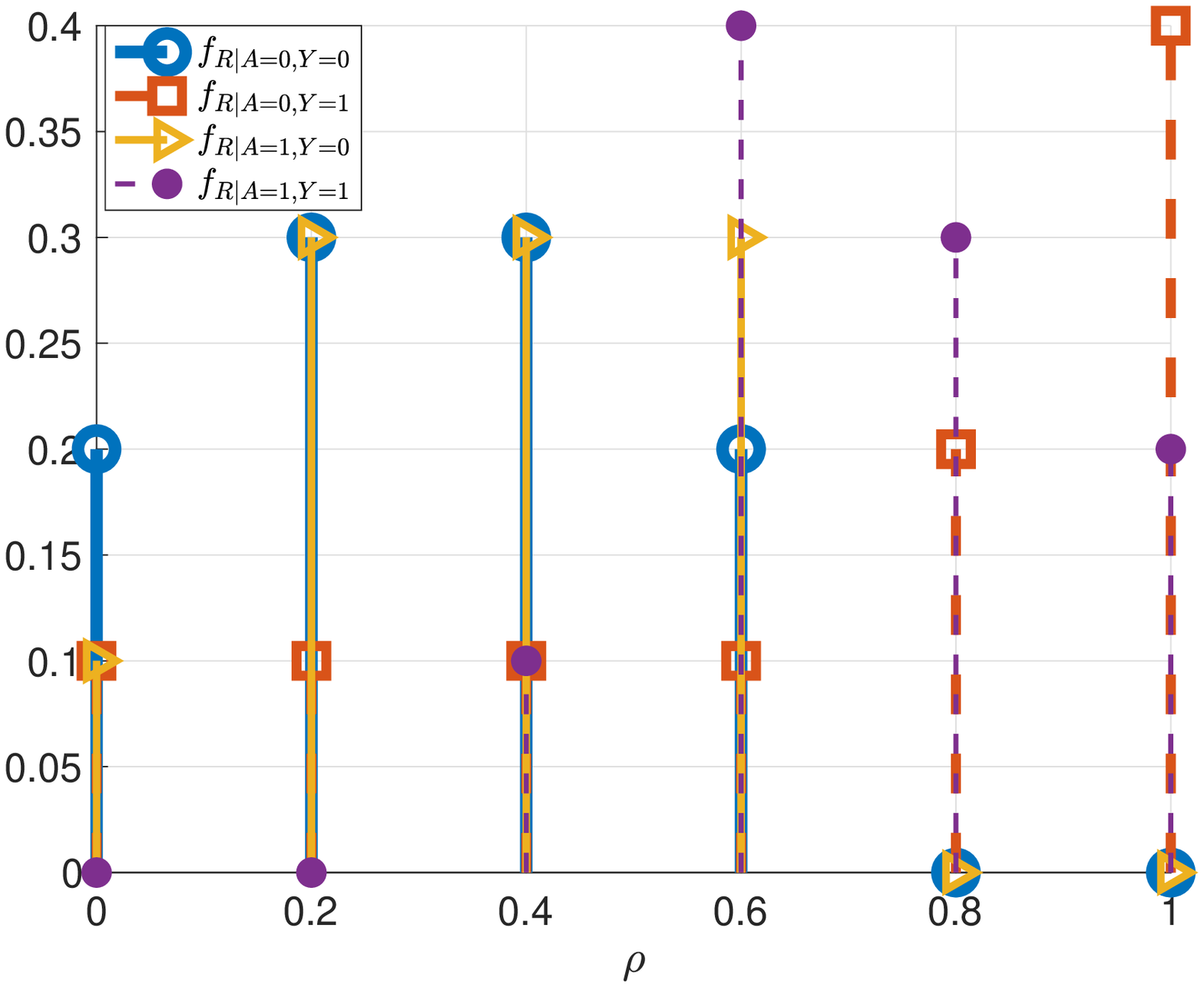}
		\caption{PMF of score $R$ conditional on $Y$ and $A$.}
		\label{fig:1}
	\end{minipage}%
	\hfill 
	\begin{minipage}{.31\textwidth}
		\centering
		\includegraphics[width=1\linewidth]{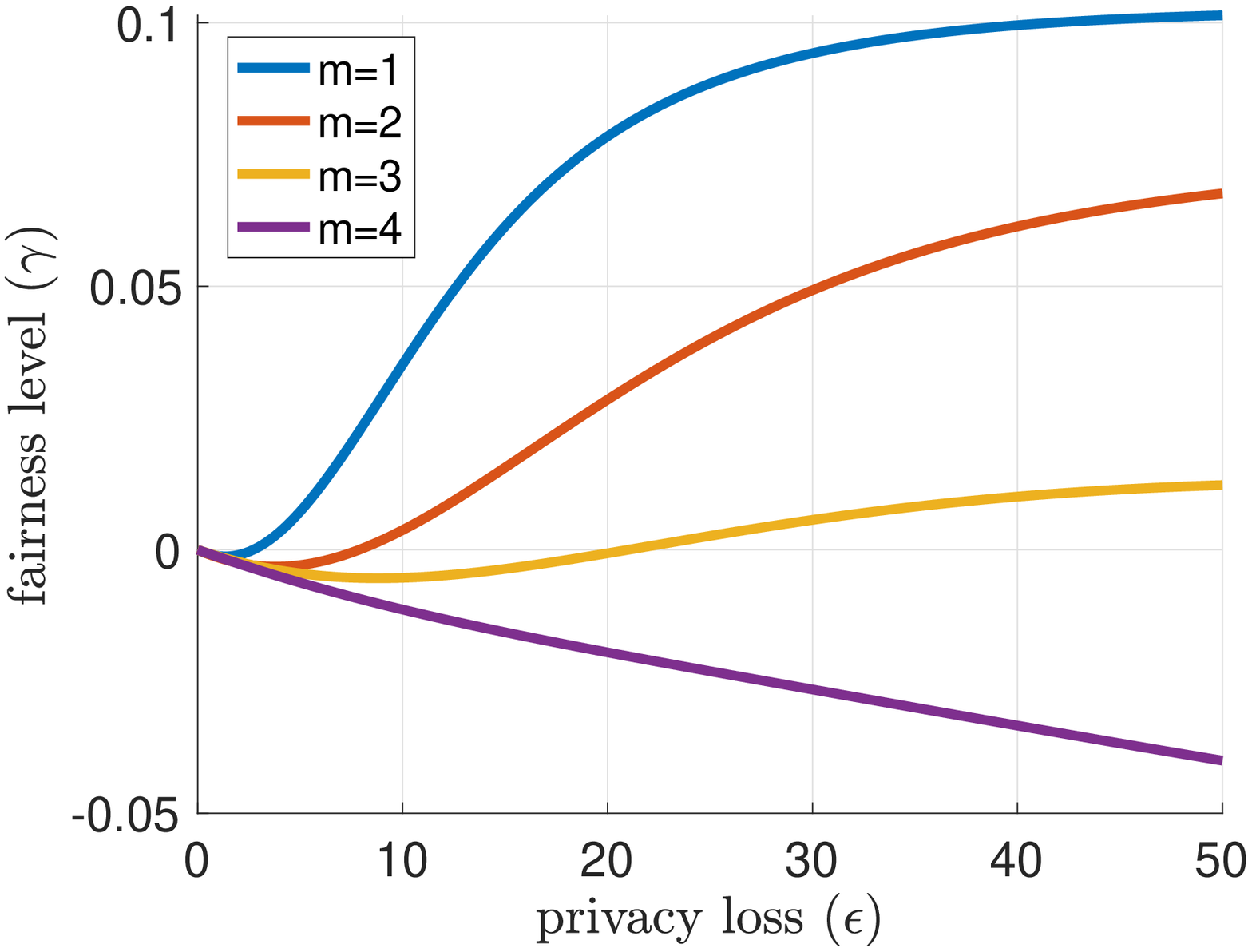}
		\caption{Fairness level attained using $\mathscr{A}_{\epsilon}(\cdot)$ and $\mathscr{B}_{\epsilon}(\cdot)$ as  functions of privacy level $\epsilon$. In this example, the perfect fairness is achievable if $m \in \{1,2,3\}$.}
		\label{fig:2}
	\end{minipage}
	\hfill 
	\begin{minipage}{.31\textwidth}
		\centering
		\includegraphics[width=1\linewidth]{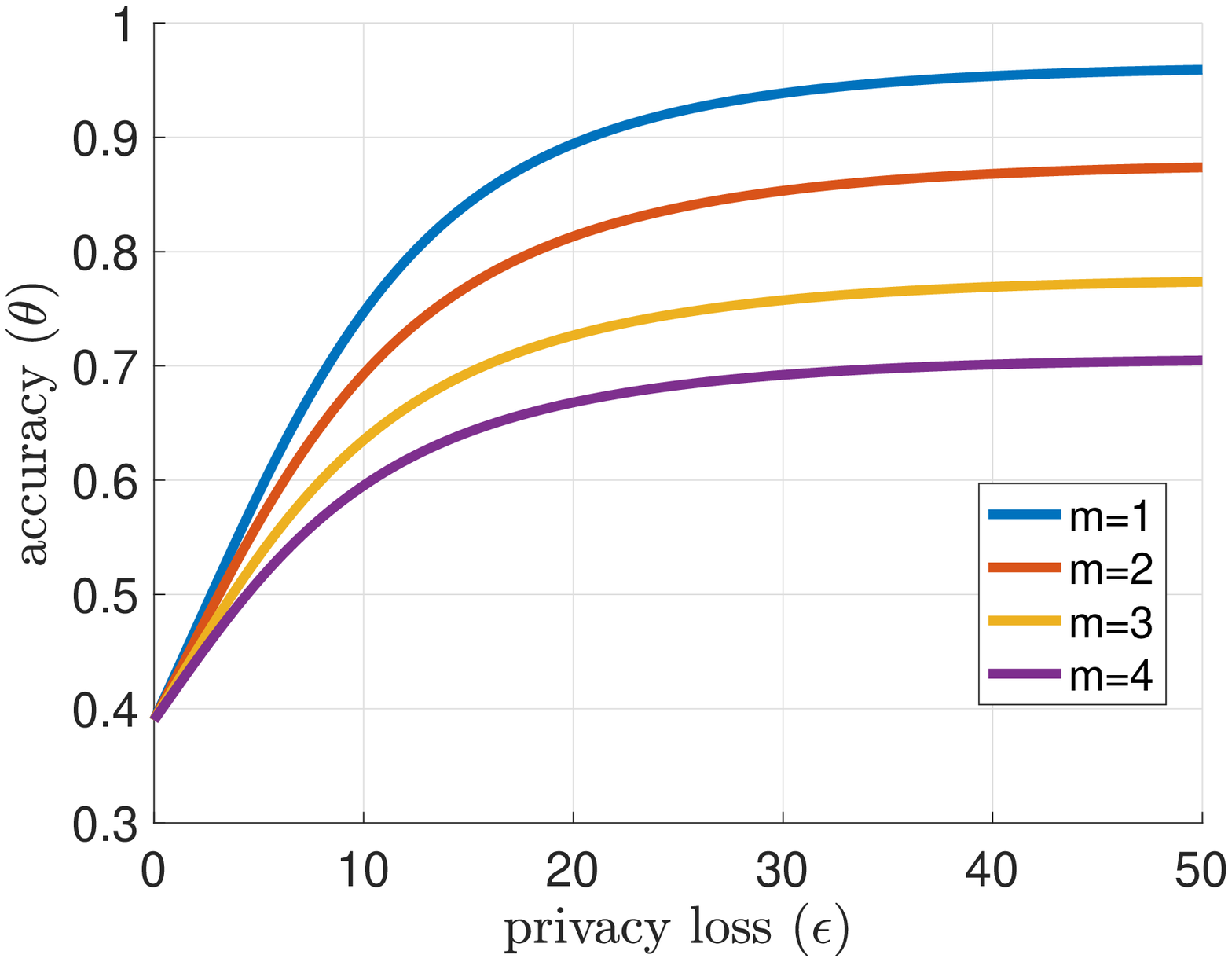}
		\caption{Accuracy level of the algorithms $\mathscr{A}_{\epsilon}(D)$ and $\mathscr{B}_{\epsilon}(D)$ as a function of privacy level $\epsilon$. 
		It shows that the accuracy is increasing in $\epsilon$. }
		\label{fig:3}
	\end{minipage}
\end{figure*}
\textbf{Case study 1: Synthetic data} 

To evaluate the fairness of algorithms $\mathscr{A}_\epsilon(\cdot)$ and  $\mathscr{B}_\epsilon(\cdot)$,
we consider a scenario where the qualification scores are generated randomly based on a distribution shown in Figure \ref{fig:1}. In this scenario, $n = 10$, $\Pr\{A=0\} = 0.1, \Pr\{Y=0|A=0\} = 0.7$, and $ \Pr\{Y=0|A=1\} = 0.6$. Figure \ref{fig:2} illustrates the fairness level \xr{$\gamma(\epsilon)$} of algorithms $\mathscr{A}_{\epsilon}(\cdot)$ and $\mathscr{B}_\epsilon(\cdot)$ as a function of $\epsilon$. In this case, both conditions in Theorem \ref{theo:fair} and Theorem \ref{theo:M} are satisfied when $m\in \{1,2,3\}$ (See the appendix for details). 
As a result, we can find privacy parameter $\epsilon_o$ at which the exponential mechanism is perfectly fair. 
Note that  conditions of Theorem \ref{theo:M} do not hold for $m=4$ (see the appendix) and  $\mathscr{B}_{\epsilon}(\cdot)$ is not perfectly fair in this case. This is confirmed in Figure 2. Figure \ref{fig:3} illustrates accuracy of 
 $\mathscr{A}_{\epsilon} (\cdot)$ and $\mathscr{B}_\epsilon (\cdot)$ as a function of privacy loss $\epsilon$. As expected, accuracy $\theta(\epsilon)$ is increasing in $\epsilon$. By comparing  Figure \ref{fig:2} and Figure \ref{fig:3},  we observe that even though improving privacy decreases accuracy, it can improve fairness.  Lastly, privacy and accuracy of the exponential mechanism under perfect fairness have been provided in Table \ref{Table}. 

\begin{table}[t!]
	\caption{Accuracy and privacy under perfect fairness. $\epsilon_o$ is a privacy parameter at which the exponential mechanism is perfectly fair. If $\epsilon_o = 0$ in this table, then the exponential mechanism with a non-zero privacy parameter cannot achieve perfect fairness. }
\centering
\begin{tabular}{ccccccc} 
\toprule
& &$\epsilon_o$ &  $\theta(\epsilon_o)$&$\theta(\infty)$& Acc. Reduction. \\
\midrule

&$m=1$ 	& 2.76& 0.50	& 0.96 &  47.92\% \\
&$m=2$ 	& 7.78 & 0.64 & 0.88 &  27.27\%  \\
&$m=3$ 	& 21.11 & 0.73 & 0.77 & 5.19\%   \\
\rot{\rlap{~Synthetic}}
&$m=4$ & 0& 0.4 &0.71 & 44.66\%\\
\midrule
&$m=1$ 	& 10.35& 0.94	& 0.97 &  3.09\%  \\
&$m=2$ 	& 22.47 & 0.94& 0.96 &  2.10\%  \\
\rot{\rlap{\hspace{-0.1cm }FICO~ ~~}}
\rot{\rlap{ \hspace{-0.4cm }\small{Fig \ref{fig:fairnessFICOWHA} \& \ref{fig:accuracyFICOWHA}}}}
&$m=3$ 	& 0 & 0.70 & 0.93 & 24.73\%   \\
&$m=4$ & 0& 0.70 & 0.90 & 22.22\%\\
\bottomrule 
\end{tabular}
\label{Table}
\vspace{-7pt}
\end{table}

\begin{figure*}[t]
	\centering
	\begin{minipage}{.30\textwidth}
		\centering
		\includegraphics[width=1\linewidth]{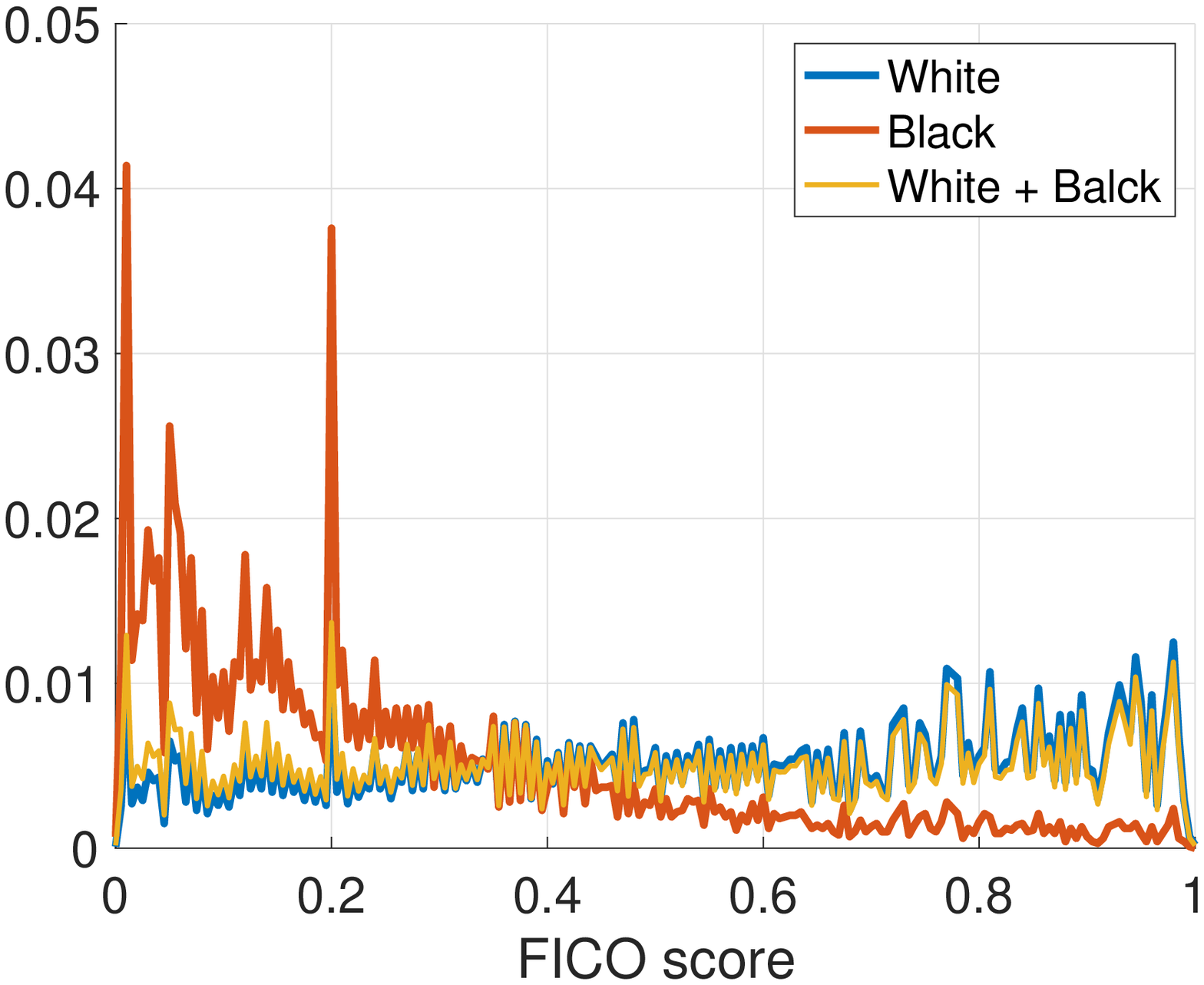}
		\caption{PMF of FICO score for \texttt{Black} and \texttt{White} social groups. }
		\label{fig:pdfFICO}
	\end{minipage}%
	\hfill 
	\begin{minipage}{.32\textwidth}
		\centering
		\includegraphics[width=1\linewidth]{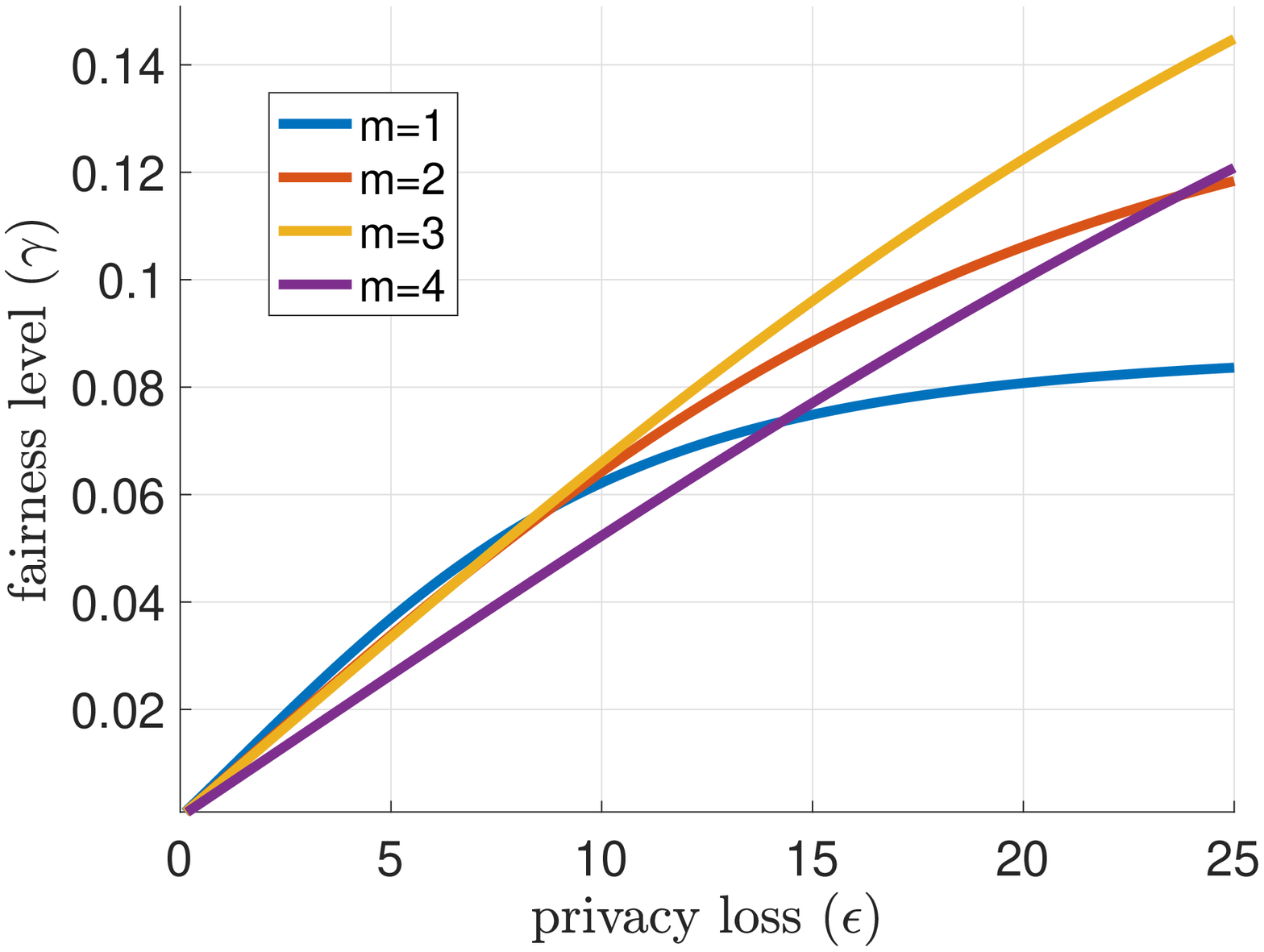}
		\caption{Fairness $\gamma(\epsilon)$ when $m$ individuals from \texttt{White} and \texttt{Black}  groups are selected using $\mathscr{A}_{\epsilon}(\cdot)$ and $\mathscr{B}_{\epsilon}(\cdot)$
		}
		\label{fig:fairnessFICO}
	\end{minipage}
	\hfill 
	\begin{minipage}{.32\textwidth}
		\centering
		\includegraphics[width=1\linewidth]{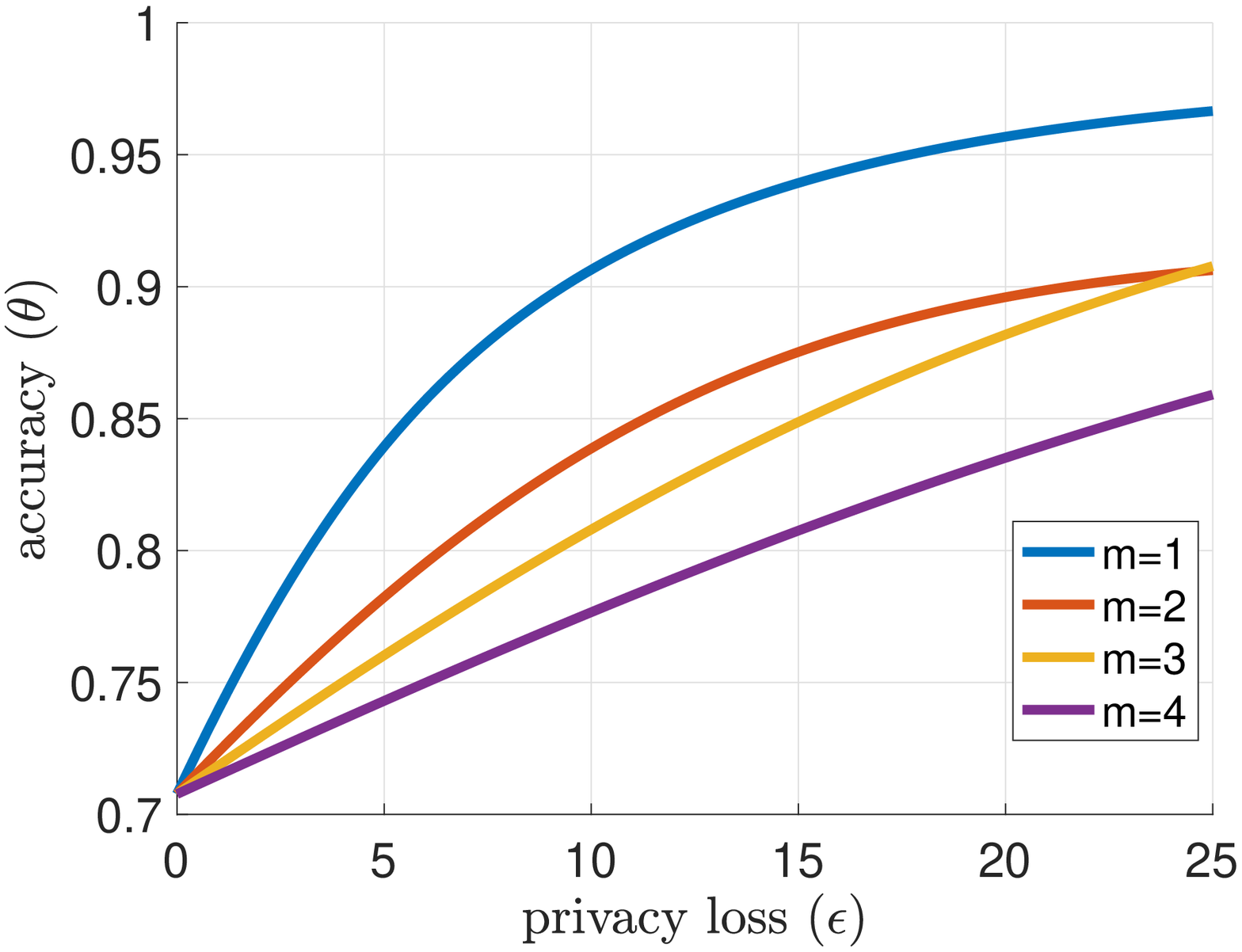}
		\caption{ Accuracy $\theta(\epsilon)$ when $m$ individuals from \texttt{White} and \texttt{Black} groups are selected using $\mathscr{A}_{\epsilon}(\cdot)$ and $\mathscr{B}_{\epsilon}(\cdot)$.
		}
		\label{fig:accuracyFICO}
	\end{minipage}
\end{figure*}
\begin{figure*}[!t]
	\centering
	\begin{minipage}{.3\textwidth}
		\centering
		\includegraphics[width=1\linewidth]{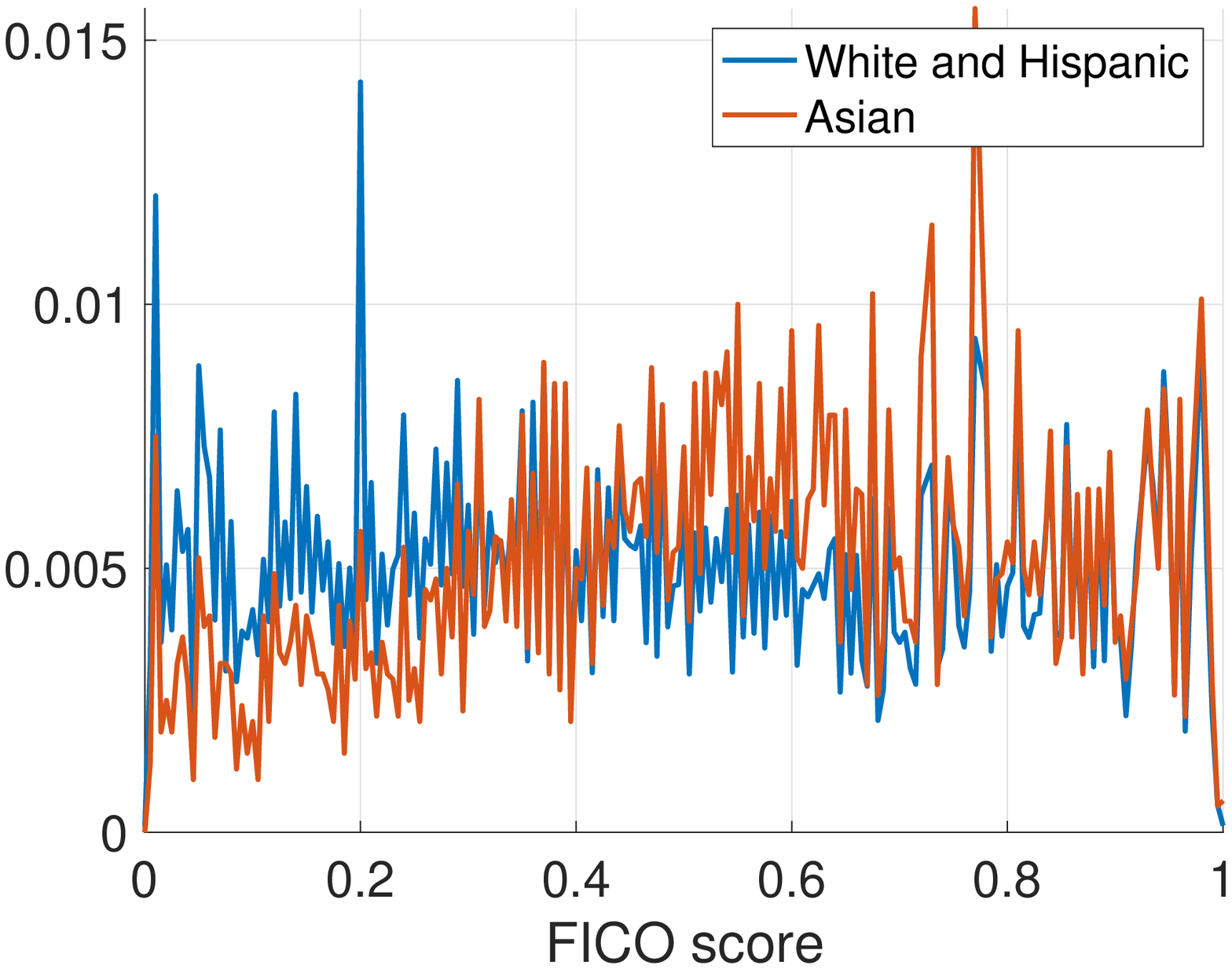}
		\caption{PMF of FICO score for \texttt{White-Hispanic} and \texttt{Asian} social groups. }
		\label{fig:pdfFICOWHA}
	\end{minipage}%
	\hfill 
	\begin{minipage}{.3\textwidth}
		\centering
		\includegraphics[width=1\linewidth]{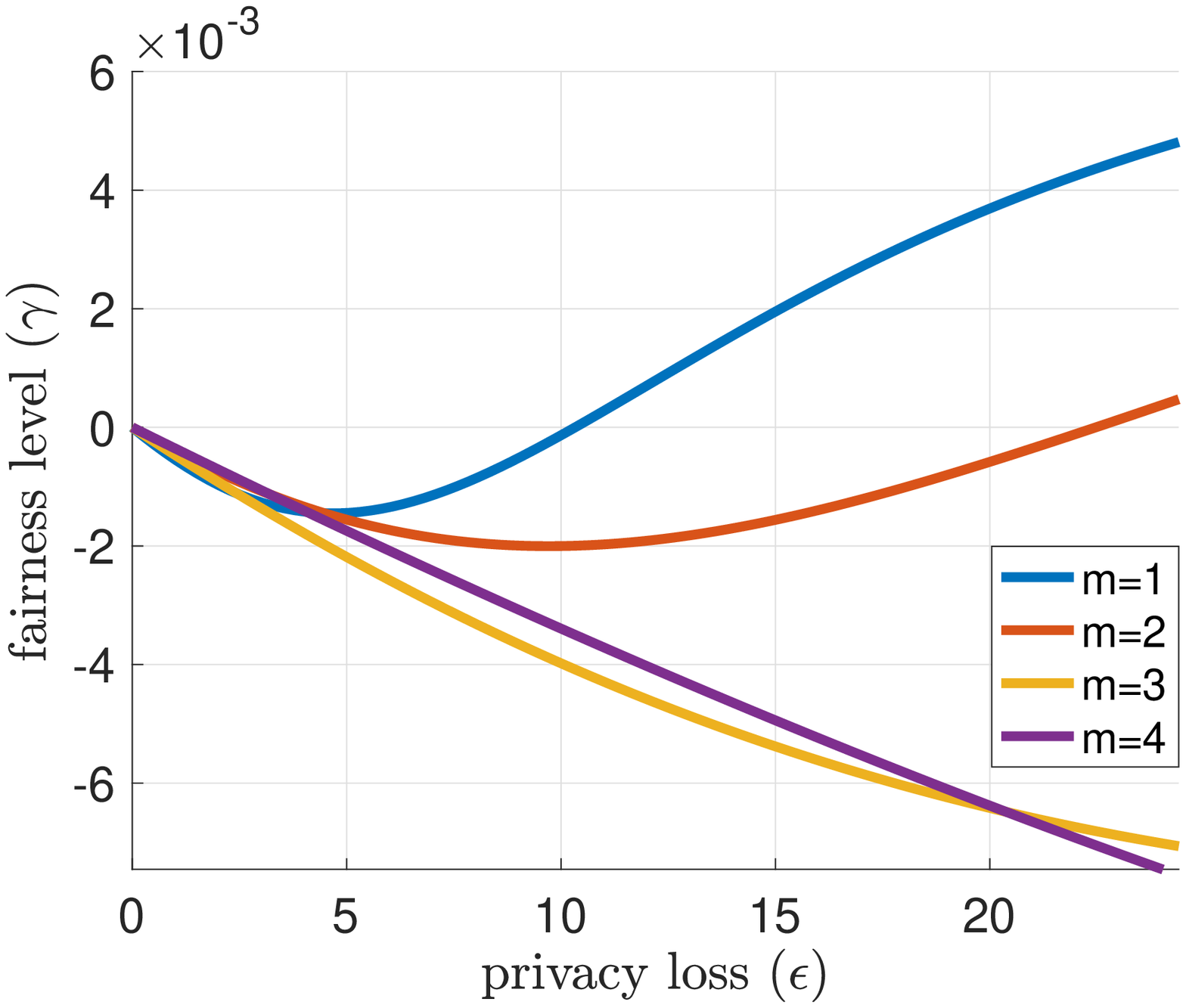}
		\caption{Fairness level $\gamma(\epsilon)$ when $m$ individuals from \texttt{White-Hispanic} and \texttt{Asian} social groups are selected using $\mathscr{A}_{\epsilon}(\cdot)$ and $\mathscr{B}_{\epsilon}(\cdot)$.
		}
		\label{fig:fairnessFICOWHA}
	\end{minipage}
	\hfill 
	\begin{minipage}{.3\textwidth}
		\centering
		\includegraphics[width=1\linewidth]{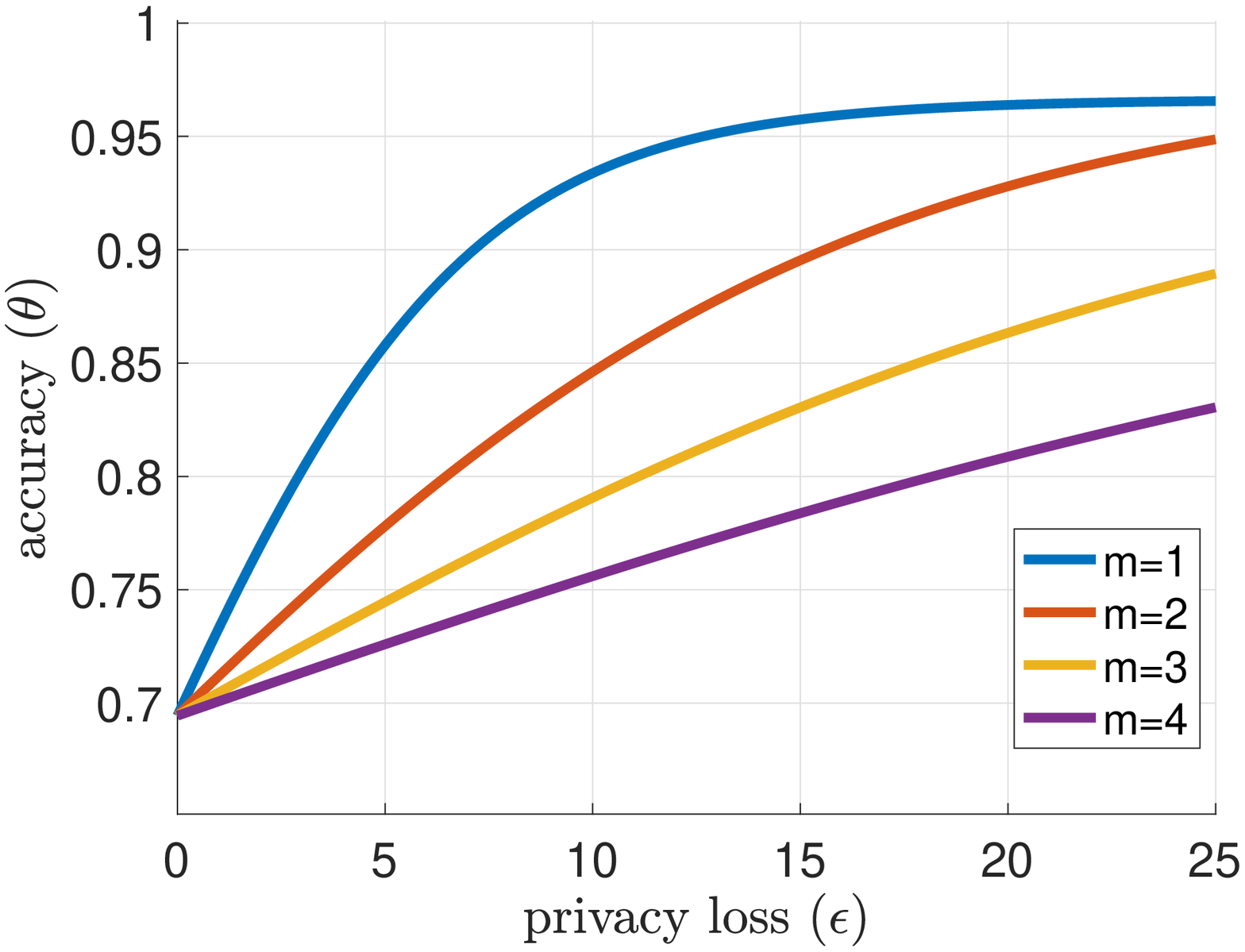}
		\caption{Accuracy level $\theta(\epsilon)$ when $m$ individuals from \texttt{White-Hispanic} and \texttt{Asian} social groups are selected using $\mathscr{A}_{\epsilon}(\cdot)$ and $\mathscr{B}_{\epsilon}(\cdot)$.
		}
		\label{fig:accuracyFICOWHA}
	\end{minipage}
\end{figure*}
\textbf{Case study 2: FICO score}

\revv{We conduct two experiments using FICO credit score dataset \cite{reserve2007report}. FICO scores are widely used in the United States to predict how likely an applicant is to pay back a loan. FICO scores are ranging from 350 to 850. However, we normalize them to range from zero to one. The FICO credit score dataset has been processed by Hardt \textit{et al.} \cite{hardt2016equality} to generate CDF and non-default rate (i.e., $\Pr(Y=1|R=\rho)$) for different social groups (Asian, White, Hispanic, and Black).}
	
	First, we consider a setting where individuals from White and Black groups are selected based on their FICO credit scores using the exponential mechanism. 
Figure  \ref{fig:pdfFICO} illustrates the PMF of FICO score for {White} and {Black} groups. It shows PMF is (approximately) decreasing  for {Black} group while is (approximately) increasing for {White} group. Moreover, the overall PMF for two groups is (approximately) uniform and remains constant. 
 As shown in Figure \ref{fig:fairnessFICO} and \ref{fig:accuracyFICO}, both accuracy $\theta(\epsilon)$ and fairness $\gamma(\epsilon)$ are increasing in $\epsilon$. Therefore, both algorithm $\mathscr{A}_{\epsilon}(\cdot)$ and $\mathscr{B}_{\epsilon}(\cdot)$ cannot be perfectly fair, and the conditions in Theorem \ref{theo:fair} and \ref{theo:M} do not hold in this example. 


\newrev{In the next experiment, we combine  {White} and {Hispanic} applicants into one group and regard {Asian} applicants as the other social group. 	 Figure \ref{fig:pdfFICOWHA} illustrates the PMF of FICO score for these two groups. In this example, the conditions of Theorem \ref{theo:fair} and  Theorem \ref{theo:M}  are satisfied for $m\in\{1,2\}$. 
When $m\in\{1,2\}$, perfect fairness is achievable   for some $\epsilon_o>0$ and leads to a slight decrease in accuracy compared to  non-private algorithms  $\mathscr{A}(\cdot)$ and $\mathscr{B}(\cdot)$ (see Table \ref{Table}).}  

\section{Conclusion}\label{sec:conclusion}
In this paper, we consider a common scenario in job/loan applications where a decision-maker selects a limited number of people from an applicant pool based on their qualification scores. These scores are generated by a pre-trained supervised learning model, which may be biased against certain social groups, and the use of such a model may violate an applicant's privacy. Within this context, we investigated the possibility of using an exponential mechanism to address the \revv{privacy} and unfairness issue. We \xr{show} that this mechanism can \xr{be used} as a post-processing step to improve fairness and \revv{privacy} of the pre-trained supervised model. 
Moreover, we identified conditions under which the exponential mechanism is able to make the selection procedure  \textit{perfectly} fair.

\bibliographystyle{plain}
\bibliography{papers.bib}

\newpage 
\onecolumn
\section*{Appendix}

\subsection*{Proofs}

\textbf{Proof of Lemma \ref{lem:convegance}:}

1. Suppose random variable $R_i$ is a mapping from sample space $\Omega$ to $\mathbb{R}$. Note that $Z_i = \left\lbrace\begin{array}{ll}
0 &\mbox{if}~ R_i \neq \max_{j} R_j\\
\frac{1}{{N}_{\max}}&o.w.
\end{array}\right.$, and $ Z_{i,\epsilon} = \frac{\exp\{\epsilon\cdot \frac{R_i}{2}\}}{\sum_{j \in \mathcal{N}} \exp\{\epsilon\cdot \frac{R_j}{2}\} }$. Moreover, we have,  
$$ \lim_{\epsilon \to \infty} \frac{\exp\{\epsilon \frac{R_i(\omega_i)}{2}\}}{\sum_{j=1}^n \exp\{\epsilon \frac{R_j(\omega_j)}{2}\}} = \left\lbrace\begin{array}{ll}
0 &\mbox{if}~ R_i(\omega_i) \neq \max_{j} R_j(\omega_j)\\
\frac{1}{|\arg\max_j R_j(\omega_j)|}&o.w.
\end{array}\right., ~~~~~\forall (\omega_1,\ldots,\omega_n)\in\Omega^n,$$ where $|\arg\max_j R_j(\omega_j)|$  is the cardinality of set   $\arg\max_j R_j(\omega_j)$.  
Therefore, random variable $Z_{i,\epsilon}$ converges to $Z_i$ surely as $\epsilon \to +\infty$.

2. 
Note that $\gamma(\epsilon)$ is a continuous function because it can be written  as the summation of $(n')^n$ continuous functions (continuous in $\epsilon$), 
\begin{eqnarray}\label{eq:h}
\gamma(\epsilon) &=& 	E\left\lbrace Z_{i,\epsilon} | A_i = 0,Y_i =1\right\rbrace - E\left\lbrace Z_{i,\epsilon} | A_i = 1,Y_i =1\right\rbrace \nonumber \\&=& \sum_{(r_1,\ldots,r_n) \in \mathcal{R}^n} \frac{\exp\{\epsilon\cdot \frac{r_i}{2}\}}{\sum_{k \in \mathcal{N}} \exp\{\epsilon\cdot \frac{r_k}{2}\} } \left(\prod_{k\neq i} f_{R}(r_k)\right) (f_{R|A=0,Y=1}(r_i) - f_{R|A=1,Y=1}(r_i)),
\end{eqnarray}
where $f_R(\cdot)$ is the PMF of $R$, and $f_{R|A=a,Y=y}(\cdot)$ is the PMF of $R$ conditional on $A=a$ and $Y=y$. 

It is easy to see that  $\lim_{\epsilon\to +\infty} \gamma(\epsilon)$ exists and is given by, 
\begin{eqnarray}
\lim_{\epsilon\to +\infty} \gamma(\epsilon) 
&=&\lim_{\epsilon \to +\infty}	\sum_{(r_1,\ldots,r_n) \in \mathcal{R}^n} \frac{\exp\{\epsilon\cdot \frac{r_i}{2}\}}{\sum_{k \in \mathcal{N}} \exp\{\epsilon\cdot \frac{r_k}{2}\} } \left(\prod_{k\neq i} f_R(r_k)\right) (f_{R|A=0,Y=1}(r_i) - f_{R|A=1,Y=1}(r_i))\nonumber \\
&=&	\sum_{(r_1,\ldots,r_n) \in \mathcal{R}^n} \lim_{\epsilon \to +\infty} \frac{\exp\{\epsilon\cdot \frac{r_i}{2}\}}{\sum_{k \in \mathcal{N}} \exp\{\epsilon\cdot \frac{r_k}{2}\} } \left(\prod_{k\neq i} f_R(r_k)\right) (f_{R|A=0,Y=1}(r_i) - f_{R|A=1,Y=1}(r_i)) \label{eq:1}  \\
&=& {E}\{\lim_{\epsilon \to +\infty}Z_{i,\epsilon}|A_i=0,Y_i=1\} - {E}\{\lim_{\epsilon \to +\infty} Z_{i,\epsilon} |A_i=1,Y_i=1\}\nonumber\\
&=&{E}\{Z_i |A_i=0,Y_i=1\} - {E}\{Z_i |A_i=1,Y_i=1\}.\label{eq:2}
\end{eqnarray}  
Equation \eqref{eq:1} holds because $\gamma(\epsilon)$ is the summation of a finite number of continuous functions, and equation \eqref{eq:2} holds since $Z_{i,\epsilon}$ converges surely towards $Z_{i}$. 

\vspace{0.5cm}
\textbf{Proof of Theorem \ref{theo:fair}: }

 Suppose $ E\{R_i|A_i=0,Y_i=1\} > E\{R_i|A_i=1,Y_i=1\}$, and ${E}\{Z_i |A_i=0,Y_i=1\} <{E}\{Z_i |A_i=1,Y_i=1\}$. As mentioned before, $\gamma(0) = 0$. Since ${E}\{Z_i |A_i=0,Y_i=1\} <{E}\{Z_i |A_i=1,Y_i=1\}$, it is easy to see that $\gamma_{\infty}<0$.

Now, we calculate the derivative of function $\gamma(\epsilon)$ at $\epsilon=0$. 

\begin{eqnarray}\label{eq:derivative}
	\gamma'(\epsilon) &=& \frac{d~ \gamma(\epsilon)}{d ~\epsilon} = {E}\left\lbrace \frac{d}{d~ \epsilon }
	\frac{\exp\{\epsilon \frac{R_i}{2}\}}{\sum_{j=1}^n\exp\{\epsilon\frac{R_j}{2}\}} \middle| A_i = 0, Y_i=1 \right\rbrace - {E}\left\lbrace \frac{d}{d~ \epsilon }
	\frac{\exp\{\epsilon \frac{R_i}{2}\}}{\sum_{j=1}^n\exp\{\epsilon\frac{R_j}{2}\}} \middle| A_i = 1, Y_i=1 \right\rbrace \nonumber \\
	&=& {E}\left\lbrace 
	\frac{\frac{R_i}{2}\exp\{\epsilon \frac{R_i}{2}\} \sum_{j=1}^n\exp\{\epsilon\frac{R_j}{2}\} - \exp\{\epsilon \frac{R_i}{2}\}\sum_{j=1}^n\frac{R_j}{2}\exp\{\epsilon\frac{R_j}{2}\} }{\left (\sum_{j=1}^n\exp\{\epsilon\frac{R_j}{2}\}\right)^2} \middle| A_i = 0, Y_i=1 \right\rbrace\nonumber \\
	&-&{E}\left\lbrace 
	\frac{\frac{R_i}{2}\exp\{\epsilon \frac{R_i}{2}\} \sum_{j=1}^n\exp\{\epsilon\frac{R_j}{2}\} - \exp\{\epsilon \frac{R_i}{2}\}\sum_{j=1}^n\frac{R_j}{2}\exp\{\epsilon\frac{R_j}{2}\} }{\left (\sum_{j=1}^n\exp\{\epsilon\frac{R_j}{2}\}\right)^2} \middle| A_i = 1, Y_i=1 \right\rbrace \nonumber \\
	\gamma'(0) &=& \frac{1}{n^2}{E} \left\lbrace n\cdot \frac{R_i}{2} - \sum_{j=1}^n \frac{R_j}{2} \middle| A_i = 0, Y_i =1   \right\rbrace - \frac{1}{n^2}{E} \left\lbrace n\cdot \frac{R_i}{2} - \sum_{j=1}^n \frac{R_j}{2} \middle| A_i = 1, Y_i=1  \right\rbrace \nonumber \\ 
	&=& \frac{n-1}{2n^2} (E\{R_i|A_i=0,Y_i=1\} - E\{R_i|A_i=1,Y_i=1\}) > 0 
\end{eqnarray}

Since $\gamma(\cdot)$ is continuous and  $\gamma(0) = 0$ and $\gamma'(0) >0$, there exists $\overline{\epsilon} $ such that $\gamma(\overline{\epsilon}) >0$. As $\gamma(+\infty) <0 $ and $\gamma(\overline{\epsilon} ) >0$, by the intermediate value theorem there exists $\epsilon_o > \overline{\epsilon}$ such that $\gamma(\epsilon_o) = 0$, and $\mathscr{A}_{\epsilon_o}(\cdot)$ is perfectly fair.

\vspace{0.5cm}
\textbf{Proof of Theorem \ref{theo:fair2}:} 

We simplify the derivative of function $\gamma(\epsilon)$ calculated in equation \eqref{eq:derivative}.
\begin{eqnarray}\label{eq:derivative2}
\gamma'(\epsilon) 
&=& {E}\left\lbrace 
\frac{\frac{R_i}{2}\exp\{\epsilon \frac{R_i}{2}\} \sum_{j=1}^n\exp\{\epsilon\frac{R_j}{2}\} - \exp\{\epsilon \frac{R_i}{2}\}\sum_{j=1}^n\frac{R_j}{2}\exp\{\epsilon\frac{R_j}{2}\} }{\left (\sum_{j=1}^n\exp\{\epsilon\frac{R_j}{2}\}\right)^2} \middle| A_i = 0, Y_i=1 \right\rbrace\nonumber \\
&-&{E}\left\lbrace 
\frac{\frac{R_i}{2}\exp\{\epsilon \frac{R_i}{2}\} \sum_{j=1}^n\exp\{\epsilon\frac{R_j}{2}\} - \exp\{\epsilon \frac{R_i}{2}\}\sum_{j=1}^n\frac{R_j}{2}\exp\{\epsilon\frac{R_j}{2}\} }{\left (\sum_{j=1}^n\exp\{\epsilon\frac{R_j}{2}\}\right)^2} \middle| A_i = 1, Y_i=1 \right\rbrace \nonumber \\
& =&  \nonumber 
{E}\left\lbrace 
\frac{ \sum_{j=1}^n\frac{R_i-R_j}{2}\exp\{\epsilon\frac{R_i+R_j}{2}\} }{\left (\sum_{j=1}^n\exp\{\epsilon\frac{R_j}{2}\}\right)^2} \middle| A_i = 0, Y_i =1 \right\rbrace - {E}\left\lbrace 
\frac{ \sum_{j=1}^n\frac{R_i-R_j}{2}\exp\{\epsilon\frac{R_i+R_j}{2}\} }{\left (\sum_{j=1}^n\exp\{\epsilon\frac{R_j}{2}\}\right)^2} \middle| A_i = 1, Y_i=1 \right\rbrace\nonumber 
\\
 &=& \sum_{(r_1,\cdots, r_n)\in \mathcal{R}^n} \frac{\sum_{j=1}^n(r_i-r_j)\exp\{\frac{\epsilon\cdot (r_i+r_j)}{2}\}}{\left( \sum_{j=1}^n \exp\{\epsilon \cdot \frac{r_j}{2}\}\right)^2} \left( \prod_{j\neq i} f_{R}(r_j) \right) \left(f_{R|A=0,Y=1}(r_i) - f_{R|A=1,Y=1}(r_i)\right) \nonumber 
\end{eqnarray}
The derivative of $\gamma(\epsilon)$ consists of $n\cdot (n')^n$ terms. Consider one of these terms associated with $\underline{\pmb{r}} = (\underline{r}_1 , \cdots, \underline{r}_n)$, where $\underline{r}_i = t$ and $\underline{r}_k = l$:  
$$ \Gamma_1 = \frac{(t-l)\exp\{\frac{\epsilon\cdot (t+l)}{2}\}}{\left( \sum_{j=1}^n \exp\{\epsilon \cdot \frac{\underline{r}_j}{2}\}\right)^2} \left( \prod_{j\neq i} f_R(\underline{r}_j) \right) \left(f_{R|A=0,Y=1}(t) - f_{R|A=1,Y=1}(t)\right) $$
Corresponding to this term, there exists another term associated with $\overline{\pmb{r}} = (\overline{r}_1 , \cdots, \overline{r}_n)$, where $\overline{r}_i = l$ and $\overline{r}_k = t$ and $\overline{r}_j = \underline{r}_j , \forall j\neq i,k $: 
$$ \Gamma_2 = \frac{(l-t)\exp\{\frac{\epsilon\cdot (t+l)}{2}\}}{\left( \sum_{j=1}^n \exp\{\epsilon \cdot \frac{\overline{r}_j}{2}\}\right)^2} \left( \prod_{j\neq i} f_R(\overline{r}_j) \right) \left(f_{R|A=0,Y=1}(l) - f_{R|A=1,Y=1}(l)\right) $$
Let $t>l$. If we show that $(t-l) f_{R}(l) (f_{R|A=0,Y=1}(t) - f_{R|A=1,Y=1}(t)) +(l-t) f_R(t) (f_{R|A=0,Y=1}(l) - f_{R|A=1,Y=1}(l))>0, \forall 1\geq t > l \geq 0 $, then $\Gamma_1 + \Gamma_2$ would be positive. Consequently,  the derivative would be positive.  It is easy to see that if $f_R(\rho)$ is strictly decreasing and $f_{R|A=0,Y=1}(\rho)- f_{R|A=1,Y=1}(\rho)$ is strictly increasing and $f_{R|A=0,Y=1}(\rho) - f_{R|A=1,Y=1}(\rho)\geq 0$  for $\rho = \rho_2,\ldots, \rho_{n'}$, then  $(t-l) f_R(l) (f_{R|A=0,Y=1}(t) - f_{R|A=1,Y=1}(t)) +(l-t) f_R(t) (f_{R|A=0,Y=1}(l) - f_{R|A=1,Y=1}(l))>0, \forall 1\geq t > l \geq 0 $, and $\gamma'(\epsilon)$ is positive, and $\gamma(\epsilon)$ is strictly increasing. Therefore, $0 = \gamma(0)<\gamma(\epsilon)<\gamma_\infty$.

\vspace{0.5cm}
\textbf{Proof of Theorem \ref{theo:nInfinity}:} 

\xr{
	As we showed in the proof of Theorem \ref{lem:convegance}, $\gamma(\epsilon)$ is continuous and $\lim_{\epsilon \to +\infty} \gamma(\epsilon)$ exists. Since $\gamma(0) = 0$ and $|\gamma_{\infty}| >0$ and $\gamma(\epsilon)$ is continuous, for any arbitrary positive value $v < |\gamma_{\infty}|$, there exists a privacy loss $\epsilon'$ such that $|\gamma(\epsilon')| = v$. This proves the theorem. 
}

\vspace{0.5cm}
\textbf{Proof of Theorem \ref{theo:acc}:} 

In order to show that $\theta(\epsilon)$ is increasing, we should show that $E\{Z_{i,\epsilon}|Y_i=1\}$ is increasing. We have, 
\begin{eqnarray}
l(\epsilon) &=& 	E\{Z_{i,\epsilon}|Y_i=1\} = \sum_{(r_1,\ldots,r_n) \in \mathcal{R}^n} \frac{\exp\{\epsilon\cdot \frac{r_i}{2}\}}{\sum_{k \in \mathcal{N}} \exp\{\epsilon\cdot \frac{r_k}{2}\} } \left(\prod_{k\neq i} f_{R}(r_k)\right) f_{R|Y=1}(r_i) \nonumber \\
l'(\epsilon) &=&  \sum_{(r_1,\ldots,r_n) \in \mathcal{R}^n}\frac{\sum_{k=1}^n\frac{r_i-r_k}{2} \exp\{\epsilon\cdot \frac{r_i+r_k}{2}\}}{(\sum_{k \in \mathcal{N}} \exp\{\epsilon\cdot \frac{r_k}{2}\})^2 } \left(\prod_{k\neq i} f_{R}(r_k)\right) f_{R|Y=1}(r_i)\nonumber 
\end{eqnarray}
In order to show that $l'(\epsilon)\geq0$, it is sufficient to show that $\frac{\rho'-\rho}{2} f_R(\rho) \cdot f_{R|Y=1}(\rho') + \frac{\rho-\rho'}{2} f_R(\rho') \cdot f_{R|Y=1}(\rho) \geq 0$, or equivalently, 
$f_R(\rho') \cdot f_{R|Y=1}(\rho) \geq f_R(\rho) \cdot f_{R|Y=1}(\rho') ~for~ \rho>\rho'$. 

We have,
\begin{eqnarray}\label{eq:eq}
	&&f_R(\rho') \cdot f_{R|Y=1}(\rho) \geq f_R(\rho) \cdot f_{R|Y=1}(\rho') \nonumber\\ 
	& \Longleftrightarrow &f_{R|Y=1}(\rho) (\Pr(Y=0) f_{R|Y=0}(\rho')+ \Pr(Y=1) f_{R|Y=1}(\rho'))\nonumber \\&& \geq 	f_{R|Y=1}(\rho') (\Pr(Y=0) f_{R|Y=0}(\rho)+ \Pr(Y=1) f_{R|Y=1}(\rho))\nonumber\\
	& \Longleftrightarrow& f_{R|Y=1}(\rho)f_{R|Y=0}(\rho') \geq f_{R|Y=1}(\rho')f_{R|Y=0}(\rho)
\end{eqnarray}
Equation \eqref{eq:eq} holds under Assumption \ref{assump1}. Therefore, $\theta(\epsilon)$ is an increasing function.


\vspace{0.5cm}
\textbf{Proof of Corollary \ref{corollary}:}

		By Theorem \ref{theo:acc}, $\theta(\epsilon)$ is an increasing function. Therefore, if $\gamma(\epsilon_{\max}) \leq \gamma_{\max}$, then $\epsilon^* =  \epsilon_{\max}$. Otherwise, the decision maker has to find the largest privacy loss with fairness level $\gamma_{\max}$. Note that $ \{\epsilon\leq \epsilon_{\max}| \gamma(\epsilon) = \gamma_{\max} \}$ is a non-empty set because $\gamma(\epsilon)$ is a continuous function and $\gamma(0) = 0$ and $\gamma(\epsilon_{\max}) > \gamma_{\max} $. Therefore, the largest value in set $ \{\epsilon\leq \epsilon_{\max}| \gamma(\epsilon) = \gamma_{\max} \}$ would be the solution to optimization problem \eqref{eq:opt}. 

\vspace{0.5cm}
\textbf{Proof of Theorem \ref{theo:M}: } 

Similar to Lemma \ref{lem:convegance}, we can show that $W_{i,\epsilon}$ converges surly towards $W_i$. Moreover, we can show that $\lim_{\epsilon\to \infty} \gamma(\epsilon) = {E}\{W_i|A_i=0,Y_i=1 \} - {E}\{W_i|A_i=1,Y_i=1\}$.

Now suppose that ${E}\{W_i|A_i=0,Y_i=1 \} <{E}\{W_i|A_i=1,Y_i=1\}$ and $E\{R_i | A_i = 0, Y_i = 1\} > E\{R_i|A_i=1, Y_i=1 \}$. Note that, 
$$\gamma_{\infty} =  {E}\{W_i|A_i=0,Y_i=1 \} - {E}\{W_i|A_i=1,Y_i=1\} <0$$

Next we calculate $\gamma'(0)$. Let $\mathcal{S}_i= \{\mathcal{G}|\mathcal{G}\in \mathcal{S}, i\in \mathcal{G}\}$.  
Similar to equation \eqref{eq:derivative}, we have, 
  \begin{eqnarray}\label{eq:derivative3}
  \gamma'(\epsilon)  & =&  \nonumber {E}\left\lbrace \frac{ \sum_{\mathcal{G}\in \mathcal{S}_i, \mathcal{G}'\in \mathcal{S}}\frac{v(\mathcal{G},\mathbf{D})-v(\mathcal{G}',\mathbf{D})}{2}\exp\{\epsilon\frac{v(\mathcal{G},\mathbf{D})+v(\mathcal{G}',\mathbf{D})}{2}\} }{\left (\sum_{\mathcal{G}'\in \mathcal{S}}\exp\{\epsilon\frac{v(\mathcal{G}',\mathbf{D})}{2}\}\right)^2} \middle| A_i=0, Y_i=1 \right\rbrace  \nonumber \\&-&   {E}\left\lbrace 
\frac{ \sum_{\mathcal{G}\in \mathcal{S}_i, \mathcal{G'}\in \mathcal{S}}\frac{v(\mathcal{G},\mathbf{D})-v(\mathcal{G}',\mathbf{D})}{2}\exp\{\epsilon\frac{v(\mathcal{G},\mathbf{D})+v(\mathcal{G}',\mathbf{D})}{2}\} }{\left (\sum_{\mathcal{G}\in S}\exp\{\epsilon\frac{v(\mathcal{G}',\mathbf{D})}{2}\}\right)^2} \middle| A_i= 1, Y_i=1 \right\rbrace \\
 \gamma(0)&=&   {E}\left\lbrace 
  \frac{ \sum_{\mathcal{G}\in \mathcal{S}_i, \mathcal{G'}\in \mathcal{S}}\frac{v(\mathcal{G},\mathbf{D})-v(\mathcal{G}',\mathbf{D})}{2}}{\left ({{n \choose m}}\right)^2} \middle| A_i = 0,Y_i=1 \right\rbrace  \nonumber -  {E}\left\lbrace 
  \frac{ \sum_{\mathcal{G}\in \mathcal{S}_i, \mathcal{G'}\in \mathcal{S}}\frac{v(\mathcal{G},\mathbf{D})-v(\mathcal{G'},\mathbf{D})}{2}}{\left ({{n \choose m}}\right)^2} \middle| A_i =0, Y_i = 1 \right\rbrace  \nonumber \\&=& \frac{{n-1 \choose m-1} {n-1 \choose m}}{2m {n \choose m}^2}  (E(R_i | A_i = 0, Y_i = 1) - E(R_i|A_i=1, Y_i=1 )) > 0,
  \end{eqnarray}

  Since $\gamma(0) = 0$ and $\gamma'(0) >0$,  there exists $\overline{\epsilon} >0$ such that $\gamma(\overline{\epsilon}) >0 $. By the Intermediate Value Theorem, there exists $\epsilon_o>\overline{\epsilon}$ such that $\gamma(\epsilon_o) = 0$.

\subsection*{ For any $m\geq 1$, accuracy $\theta(\epsilon)$ is increasing under Assumption 1}
In Theorem \ref{theo:acc}, we showed that accuracy $\theta(\epsilon)$ is increasing for $m=1$. Here we show that $\theta(\epsilon)$ is increasing for all $m\geq 1$. In order to do so, we need to prove that $E\left\lbrace W_{i,\epsilon}|Y_i=1\right\rbrace $ is increasing in $\epsilon$.
\begin{eqnarray*}
l(\epsilon) &=& 	E\left\lbrace W_{i,\epsilon}|Y_i=1\right\rbrace = \sum_{\mathcal{G}\in \mathcal{S}_i}E\left\lbrace\frac{\exp\{\epsilon\cdot \frac{\sum_{j\in \mathcal{G} }R_j}{2}\}}{\sum_{\mathcal{G'}\in \mathcal{S}} \exp\{\epsilon \cdot \frac{\sum_{j\in \mathcal{G'} R_j}}{2}\}}  \middle| Y_i=1\right\rbrace\\
l'(\epsilon) &=& 	E\left\lbrace W_{i,\epsilon}|Y_i=1\right\rbrace = \sum_{\mathcal{G}\in \mathcal{S}_i, \mathcal{G''}\in \mathcal{S}}  E\left\lbrace\frac{\frac{\sum_{j\in \mathcal{G}}R_j  -\sum_{j\in \mathcal{G''}} R_j}{2}\exp\{\epsilon\cdot \frac{\sum_{j\in \mathcal{G}}R_j  +\sum_{j\in \mathcal{G''}} R_j}{2}\}}{\sum_{\mathcal{G'}\in \mathcal{S}} \exp\{\epsilon \cdot \frac{\sum_{j\in \mathcal{G'} R_j}}{2}\}} \middle| Y_i=1 \right\rbrace\\
(R_j ~'s~ are ~ i.i.d.)  &=&  \sum_{\mathcal{G}\in \mathcal{S}_i, \mathcal{G''}\in \mathcal{S}\backslash \mathcal{S}_i}  E\left\lbrace\frac{\frac{\sum_{j\in \mathcal{G}}R_j  -\sum_{j\in \mathcal{G''}} R_j}{2}\exp\{\epsilon\cdot \frac{\sum_{j\in \mathcal{G}}R_j  +\sum_{j\in \mathcal{G''}} R_j}{2}\}}{\sum_{\mathcal{G'}\in \mathcal{S}} \exp\{\epsilon \cdot \frac{\sum_{j\in \mathcal{G'} R_j}}{2}\}} \middle| Y_i=1 \right\rbrace
\end{eqnarray*}

Note that since $R_j$'s are i.i.d,  $E\left\lbrace\frac{\frac{R_k -R_{k'}}{2}\exp\{\epsilon\cdot \frac{\sum_{j\in \mathcal{G}}R_j  +\sum_{j\in \mathcal{G''}} R_j}{2}\}}{\sum_{\mathcal{G'}\in \mathcal{S}} \exp\{\epsilon \cdot \frac{\sum_{j\in \mathcal{G'} R_j}}{2}\}} \middle| Y_i=1 \right\rbrace = 0, for ~ k,k' \neq i, k\in \mathcal{G} \in S, k'\in \mathcal{G''}\in \mathcal{S}\backslash \mathcal{S}_i, k'\notin \mathcal{G}, k\notin \mathcal{G''}$. We have, 
\begin{eqnarray*}
l'(\epsilon)&=& \sum_{\mathcal{G}\in \mathcal{S}_i, \mathcal{G''}\in \mathcal{S}\backslash \mathcal{S}_i}  E\left\lbrace\frac{\frac{R_i  - R_k}{2}\exp\{\epsilon\cdot \frac{\sum_{j\in \mathcal{G}}R_j  +\sum_{j\in \mathcal{G''}} R_j}{2}\}}{\sum_{\mathcal{G'}\in \mathcal{S}} \exp\{\epsilon \cdot \frac{\sum_{j\in \mathcal{G'} R_j}}{2}\}} \middle| Y_i=1 \right\rbrace  , k\in \mathcal{G''}, k\notin \mathcal{G}\\
l'(\epsilon) &=&   \sum_{\mathcal{G}\in \mathcal{S}_i, \mathcal{G''}\in \mathcal{S}\backslash \mathcal{S}_i} \sum_{(r_1,\ldots,r_n)\in \mathcal{R}^n} \frac{\frac{r_i  - r_k}{2}\exp\{\epsilon\cdot \frac{\sum_{j\in \mathcal{G}}r_j  +\sum_{j\in \mathcal{G''}} r_j}{2}\}}{\sum_{\mathcal{G'}\in \mathcal{S}} \exp\{\epsilon \cdot \frac{\sum_{j\in \mathcal{G'} r_j}}{2}\}}\left(\prod_{j\neq i} f_{R}(r_j)\right) f_{R|Y=1}(r_i)
\end{eqnarray*}
In order to show that $l'(\epsilon)>0$, it is sufficient to show that $\frac{\rho'-\rho}{2} f_R(\rho) \cdot f_{R|Y=1}(\rho') + \frac{\rho-\rho'}{2} f_R(\rho') \cdot f_{R|Y=1}(\rho) > 0$, or equivalently 
$f_R(\rho') \cdot f_{R|Y=1}(\rho) > f_R(\rho) \cdot f_{R|Y=1}(\rho)$ for $\rho>\rho'$.  The rest of the proof is similar to the proof of Theorem \ref{theo:acc} and is omitted to avoid repetition. 

\subsection*{Notes on the numerical example}

\textbf{Details about case study 1 (Figures 1, 2, 3)}

In this part, we investigate whether the conditions in Theorems \ref{theo:fair} and \ref{theo:M} hold or not. First, we find mean value $E(Z_i|Y_i=1,A_i=a)$ and $E(W_i|Y_i=1,A_i=a)$ numerically. In order to do so, we generate  10000 samples of $R_1,\cdots,R_n$  conditional on $Y_i =1 $ and $A_i=a$ and calculate samples of $Z_i$ and $W_i$ and their empirical mean. We have, 

\begin{table}[!htbp]
	\centering
	\caption{Expected values $Z_i$ and $W_i$  conditional on $Y_i=1$ and $A_i=a$}\label{Table2}
	\begin{tabular}{ cccc } 
		\toprule & $E(Z_i|Y_i=1,A_i=0)$ & $E(Z_i|Y_i=1,A_i=1)$ & $E(Z_i|Y_i=1,A_i=0) - E(Z_i|Y_i=1,A_i=1)$\\  \midrule 
		 $m=1$ & 0.3426 & 0.2392 & 0.1035\\ \midrule
		& $E(W_i|Y_i=1,A_i=0)$ & $E(W_i|Y_i=1,A_i=1)$ &  $E(W_i|Y_i=1,A_i=0)-E(W_i|Y_i=1,A_i=1)$\\  \midrule 
		$m=2$ & 0.5213 & 0.4337 & 0.0875\\ 
		$m=3$ & 0.5941 & 0.5705 & 0.0236 \\ 
		$m=4$ & 0.5821  &0.6196& -0.0375\\
		\bottomrule
	\end{tabular}
\end{table}

Based on the proof of Theorem \ref{theo:fair} and Theorem \ref{theo:M}, the last column of Table \ref{Table2} is equal to $\lim_{\epsilon \to +\infty} \gamma(\epsilon)$ (it can be seen in figure \ref{fig:2}). Next, we calculate $E(R_i|Y_i=1, A_i=a)$ using the definition: $ E(R_i|Y_i=1, A_i=a) = \sum_{\rho \in \mathcal{R}} f_{R|Y=1,A=a}(\rho) \cdot \rho $. We have, 
	\begin{table}[!htbp]
			\centering
		\caption{Expected values $R_i$ conditional on $Y_i=1$ and $A_i=a$}\label{Table3}
		\begin{tabular}{ ccc } 
			\toprule $E(R_i|Y_i=1,A_i=0)$ & $E(R_i|Y_i=1,A_i=1)$ & $E(R_i|Y_i=1,A_i=0) - E(R_i|Y_i=1,A_i=1)$\\  \midrule 
		 0.68 & 0.72 & -0.04\\ \bottomrule
		\end{tabular}
	\end{table} 

Note that $E(R_i|Y_i=1,A_i=0) - E(Z_i|Y_i=1,A_i=1) $ and $\gamma'(0)$ have the same sign.  Given Table \ref{Table2} and Table \ref{Table3}, we can see that the conditions  in Theorems \ref{theo:fair} and \ref{theo:M} hold only for $m\in \{1,2,3\}$, and $\gamma(\epsilon)$ crosses zero when $m = 1,2,3$ (See Figure \ref{fig:2}).  

Figure \ref{fig:3} illustrates accuracy $\theta(\epsilon)$. This figure implies that accuracy is increasing in this case. We noticed that Assumption \ref{assump1} does not hold when $\rho = 0.2$ and $\rho' = 0$. In other words, we have, 
\begin{eqnarray*}
	&& \frac{\Pr(R_i = 0.2|Y_i=1)}{\Pr(R_i = 0|Y_i=1)}    = \frac{0.01}{0.01}<  \frac{\Pr(R_i = 0.2|Y_i=0)}{\Pr( R_i=0|Y_i=0)} =  \frac{0.3}{0.11}
\end{eqnarray*} 


Even though Assumption \ref{assump1}  does not hold, we can see $\theta(\epsilon)$ is increasing. Therefore, Assumption \ref{assump1} is a sufficient condition for $\theta(\epsilon)$ to be increasing.

\vspace{0.5cm}
\textbf{Details about case study 2 (Figures 4, 5, 6)}

In this part, we provide details about the experiment on the FICO scores of Black and White social groups. \newrev{In this example. $\Pr\{A=\texttt{White} \}  = 1- \Pr\{A=\texttt{Black} \} = 0.88$.} By checking Table \ref{Table4} and \ref{Table5}, we can see that the conditions of Theorem \ref{theo:fair} do not hold. Again, the forth column of Table \ref{Table4} is equal to $\lim_{\epsilon \to +\infty}\gamma(\epsilon)$, and the sign of the last column in Table \ref{Table5} is the same as the sign of $\gamma'(0)$ (Figure \ref{fig:fairnessFICO} confirms that). 

\begin{table}[!htbp]
		\centering
	\caption{Expected values $Z_i$ and $W_i$  conditional on $Y_i=1$ and $A_i=a$}\label{Table4}
	\begin{tabular}{ cccc } 
		\toprule& $E(Z_i|Y_i=1,A_i=0)$ & $E(Z_i|Y_i=1,A_i=1)$ & $E(Z_i|Y_i=1,A_i=0) - E(Z_i|Y_i=1,A_i=1)$\\  \midrule 
	 $m=1$ & 0.1425 & 0.0486  & 0.0939\\ \midrule
		& $E(W_i|Y_i=1,A_i=0)$ & $E(W_i|Y_i=1,A_i=1)$ &  $E(W_i|Y_i=1,A_i=0)-E(W_i|Y_i=1,A_i=1)$\\  \midrule 
		$m=2$ & 0.2851 & 0.1099 & 0.1752\\ 
		$m=3$ & 0.3765  & 0.1743 & 0.2022 \\ 
		$m=4$ & 0.4142  &0.2216& 0.1925\\
		\bottomrule
	\end{tabular}
\end{table}

	\begin{table}[!htbp]
			\centering
		\caption{Expected values $R_i$ conditional on $Y_i=1$ and $A_i=a$}\label{Table5}
		\begin{tabular}{ ccc } 
			\toprule $E(R_i|Y_i=1,A_i=0)$ & $E(R_i|Y_i=1,A_i=1)$ & $E(R_i|Y_i=1,A_i=0) - E(R_i|Y_i=1,A_i=1)$\\  \midrule 
		 0.6445 & 0.4694 & 0.1751\\ \bottomrule
		\end{tabular}
	\end{table} 

\vspace{0.5cm}
\textbf{Details about case study 2 (Figures 7, 8, 9)} 

\newrev{In this part, we provide more details about the experiment on the FICO scores of White, Hispanic and Asian community. In this example, $\Pr\{A= \texttt{White} \text{ or } \texttt{Hispanic}\} =  1 - \Pr\{A= \texttt{Asian}\}  = 0.9634$. Since the dataset provided by \cite{hardt2016equality} includes PMF functions $\Pr\{R=\rho|A = \texttt{White}\} $ and $\Pr\{R=\rho|A = \texttt{Hispanic}\}$, we calculated PMF function $\Pr\{R = \rho|A= \texttt{White} \text{ or } \texttt{Hispanic}\} $ under assumption that $\Pr\{A =\texttt{White} | A= \texttt{White} \text{ or } \texttt{Hispanic}\} =0.64 $.}

Figure \ref{fig:fairnessFICOWHA}  shows $\gamma(\epsilon)$ crosses  zero when $m=1, 2$. The reason that $\gamma(\epsilon)$ crosses zero for $m=1,2$ is that the conditions in Theorem \ref{theo:M} holds. Table \ref{Table6} and Table \ref{Table7} provides more details to check the condition in Theorem \ref{theo:fair} and Theorem \ref{theo:M}. 
\begin{table}[!htbp]
		\centering
	\caption{Expected values $Z_i$ and $W_i$  conditional on $Y_i=1$ and $A_i=a$}\label{Table6}
	\begin{tabular}{ cccc } 
		\toprule & $E(Z_i|Y_i=1,A_i=0)$ & $E(Z_i|Y_i=1,A_i=1)$ & $E(Z_i|Y_i=1,A_i=0) - E(Z_i|Y_i=1,A_i=1)$\\  \midrule 
		 $m=1$ & 0.1462 & 0.1389  & 0.0073\\ \midrule
		& $E(W_i|Y_i=1,A_i=0)$ & $E(W_i|Y_i=1,A_i=1)$ &  $E(W_i|Y_i=1,A_i=0)-E(W_i|Y_i=1,A_i=1)$\\ 	
		 \midrule 
		$m=2$ & 0.2781 & 0.2778 & 0.0003\\ 
		$m=3$ & 0.3635  & 0.3745 & -0.0110 \\ 
		$m=4$ & 0.4072  &0.4315 & -0.0243\\
		\bottomrule
	\end{tabular}
\end{table} 

	\begin{table}[!htbp]
			\centering
		\caption{Expected values $R_i$ conditional on $Y_i=1$ and $A_i=a$}\label{Table7}
		\begin{tabular}{ ccc } 
			\toprule $E(R_i|Y_i=1,A_i=0)$ & $E(R_i|Y_i=1,A_i=1)$ & $E(R_i|Y_i=1,A_i=0) - E(R_i|Y_i=1,A_i=1)$\\ 
			 \midrule   0.6088
			&     0.6213
			& -0.0125\\ \bottomrule
		\end{tabular}
	\end{table} 

\subsection*{Extension of our result to other score functions}

In section \ref{sec:choosingM}, we considered the following score function for the exponential mechanism,
\begin{equation*}
v(\mathcal{G},\mathbf{D}) = 	\frac{1}{m}\sum_{j\in \mathcal{G}} R_j.
\end{equation*}

It is worth mentioning that other choices of score functions may be in the decision-maker's interest. For instance, he may choose one of the following score functions, 
\begin{eqnarray}\label{eq:scoreFunction}
v(\mathcal{G},\mathbf{D}) = 	\sqrt{\frac{1}{m}\sum_{j\in \mathcal{G}} R^2_j}, ~ \text{or} ~ 
v(\mathcal{G},\mathbf{D}) = 	{\min_{j\in \mathcal{G}} R_j}~ \text{or} ~ v(\mathcal{G},\mathbf{D}) = 	{\max_{j\in \mathcal{G}} R_j}
\end{eqnarray}

Let algorithm $\mathscr{B}(D)$ be an algorithm which selects a set with the maximum score function. Moreover, let $\mathcal{S}_{\max} = \{\mathcal{G} \in \mathcal{S} |v(\mathcal{G},\mathbf{D}) = \max_\mathcal{G'} v(\mathcal{G'},\mathbf{D})  \}$. We define random variable $W_i$ and $W_{i,\epsilon}$ as follows,
\begin{equation*}
W_{i,\epsilon}  = \sum_{\mathcal{G} \in \mathcal{S}_i} \frac{\exp\{\epsilon \cdot  \frac{v(\mathcal{G},\mathbf{D})}{2}\}}{\sum_{\mathcal{G}'\in \mathcal{S}} \exp\{\epsilon \cdot \frac{v(\mathcal{G}',\mathbf{D})}{2}\}}
\end{equation*}

\begin{eqnarray}
W_i = \left\lbrace \begin{array}{ll}
0 & \mbox{ if } \mathcal{S}_i \cap \mathcal{S}_{\max} = \emptyset \\ 
\frac{1}{|\mathcal{S}_{\max} |} & o.w.
\end{array}\right. \nonumber
\end{eqnarray} 

Again, we can show that under a set of sufficient conditions, perfect fairness is achievable even if we use score functions defined in \eqref{eq:scoreFunction}.
\begin{theorem}

	\revv{There exists} $\epsilon_o >0$ such that $\gamma(\epsilon_o) = 0$ under $\mathscr{B}_{\epsilon}(.)$ if both of the following \revv{constraints} are satisfied:
	
	(1) $E\{W_i| A_i = a,Y_i =1  \} <E\{W_i| A_i = \neg a,Y_i =1 \}$;
	
	(2)  $E\left\lbrace v(\mathcal{G},\mathbf{D}) | A_i = a,Y_i =1\right\rbrace > E\left\lbrace v(\mathcal{G},\mathbf{D}) | A_i = \neg a,Y_i =1\right\rbrace$,\text{ for }$\mathcal{G} \in \mathcal{S}_i $ .
\end{theorem}
Since the proof is similar to the proof of Theorem \ref{theo:M}, we do not repeat the proof here. 
\newrev{
\subsection*{Privacy guarantee with respect to the training dataset}
In this paper, we studied the privacy guarantee with respect to the data of applicants. Note that an attacker can infer some information about data $\mathbf{D} = (X_1...X_n)$ by observing the selection outcome if we do not use the exponential mechanism. 

In addition to the privacy gaurantee for the applicants, our selection mechanism can provide a privacy guarantee for the training dataset used to train function $r(\cdot)$. Define $D_{all} = [D_{train},D]$ a new dataset including both training dataset $D_{train}$ and dataset of applicants $D$. Consider an exponential mechanism with a score function defined w.r.t. $D_{all}$, i.e., $v(i,D_{all}) = r(x_i)$. Note that $v(i,D_{all})$ depends on $D_{train}$ via supervised learning model $r(.)$. Because $r(x)\in [0,1]$ for all $D_{train}$ and $x$, the sensitivity of score function w.r.t. $D_{all}$ is 1, implying that the privacy of training dataset is preserved through the exponential mechanism. In other words, changing any single data point in $D_{train}$ cannot change the selection outcome significantly. }

\subsection*{Demographic parity fairness notion}

Demographic parity fairness notion in classification implies that the positive rate should be independent of sensitive attribute \cite{dwork2012fairness}. In other words, if $\hat{Y}$ is the predicted label, under demographic parity fairness constraint, $\hat{Y}$ should satisfy the following,
\begin{equation*}
\Pr(\hat{Y}=1|A=1) = \Pr(\hat{Y}=1|A=0)
\end{equation*} 
In our setting, we can adopt a fairness notion similar to demographic parity. We say algorithm $\mathscr{M}(\cdot)$ is $\gamma$-fair under demographic parity if,
\begin{eqnarray}
\Pr\{K_{i} =1 | A_i=0\} - \Pr\{ K_{i} = 1| A_i = 1\} = \gamma, \nonumber  
\end{eqnarray}
where $K_i $ is a Bernoulli random variable, and $K_i=1$ is an event indicating that individual $i$ has been selected.  
Algorithm $\mathscr{A}_{\epsilon}(\cdot)$ is $\hat{\gamma}(\epsilon)$-fair under demographic parity fairness notion if  we have, 
\begin{equation*}
\hat{\gamma}(\epsilon) = E\{Z_{i,\epsilon}|A_i=0\} - E\{Z_{i,\epsilon}|A_i=1\}.
\end{equation*}
The following theorem finds conditions under which there exists $\epsilon_o>0$ such that $\hat{\gamma}(\epsilon_o) = 0$. 

\begin{theorem}\label{theo:DP}
	\revv{There exists} $\epsilon_o >0$ such that $\hat{\gamma}(\epsilon_o) = 0$ under $\mathscr{A}_{\epsilon_o}(.)$ if both of the following \revv{constraints} are satisfied:
	
	(1) $E\{Z_i|A_i = a \}  < E\{Z_i|A_i = \neg a\}$,
	
	(2) $ E \left\lbrace R_i | A_i = a\right\rbrace >  E \left\lbrace R_i | A_i =  \neg a\right\rbrace$.
\end{theorem}

\begin{proof}
The proof is very similar to the proof of Theorem \ref{theo:fair}. Suppose $ E\{R_i|A_i=0\} > E\{R_i|A_i=1\}$, and $\mathbb{E}\{Z_i |A_i=0\} <\mathbb{E}\{Z_i |A_i=1\}$. As mentioned before, $\gamma(0) = 0$. Since ${E}\{Z_i |A_i=0\} <{E}\{Z_i |A_i=1\}$, it is easy to see that $\hat{\gamma}{(+\infty)}<0$.

Now, we calculate the derivative of function $\hat{\gamma}(\epsilon)$ at $\epsilon=0$. 
\begin{eqnarray}\label{eq:derivativeDP}
\hat{\gamma}'(\epsilon) &=& \frac{d~ \hat{\gamma}(\epsilon)}{d ~\epsilon} = {E}\left\lbrace \frac{d}{d~ \epsilon }
\frac{\exp\{\epsilon \frac{R_i}{2}\}}{\sum_{j=1}^n\exp\{\epsilon\frac{R_j}{2}\}} \middle| A_i = 0 \right\rbrace - {E}\left\lbrace \frac{d}{d~ \epsilon }
\frac{\exp\{\epsilon \frac{R_i}{2}\}}{\sum_{j=1}^n\exp\{\epsilon\frac{R_j}{2}\}} \middle| A_i = 1 \right\rbrace \nonumber \\
&=& {E}\left\lbrace 
\frac{\frac{R_i}{2}\exp\{\epsilon \frac{R_i}{2}\} \sum_{j=1}^n\exp\{\epsilon\frac{R_j}{2}\} - \exp\{\epsilon \frac{R_i}{2}\}\sum_{j=1}^n\frac{R_j}{2}\exp\{\epsilon\frac{R_j}{2}\}\} }{\left (\sum_{j=1}^n\exp\{\epsilon\frac{R_j}{2}\}\right)^2} \middle| A_i = 0 \right\rbrace\nonumber \\
&-&{E}\left\lbrace 
\frac{\frac{R_i}{2}\exp\{\epsilon \frac{R_i}{2}\} \sum_{j=1}^n\exp\{\epsilon\frac{R_j}{2}\} - \exp\{\epsilon \frac{R_i}{2}\}\sum_{j=1}^n\frac{R_j}{2}\exp\{\epsilon\frac{R_j}{2}\}\} }{\left (\sum_{j=1}^n\exp\{\epsilon\frac{R_j}{2}\}\right)^2} \middle| A_i = 1 \right\rbrace \nonumber \\
\gamma'(0) &=& \frac{1}{n^2}{E} \left\lbrace n\cdot \frac{R_i}{2} - \sum_{j=1}^n \frac{R_j}{2} \middle| A_i = 0  \right\rbrace - \frac{1}{n^2}{E} \left\lbrace n\cdot \frac{R_i}{2} - \sum_{j=1}^n \frac{R_j}{2} \middle| A_i = 1  \right\rbrace \nonumber \\ 
&=& \frac{n-1}{2n^2} (E\{R_i|A_i=0\} - E\{R_i|A_i=1\}) > 0 \nonumber
\end{eqnarray}

Since $\hat{\gamma}(\cdot)$ is continuous and  $\hat{\gamma}(0) = 0$ and $\hat{\gamma}'(0) >0$, there exists $\overline{\epsilon} $ such that $\hat{\gamma}(\overline{\epsilon}) >0$. As $\hat{\gamma}(+\infty) <0 $ and $\hat{\gamma}(\overline{\epsilon} ) >0$, by the Intermediate Value Theorem there exists $\epsilon_o > \overline{\epsilon}$ such that $\hat{\gamma}(\epsilon_o) = 0$, and $\mathscr{A}_{\epsilon_o}(\cdot)$ is perfectly fair.
\end{proof}

	
%

\end{document}